\documentclass{article}

\PassOptionsToPackage{numbers}{natbib}
\usepackage[final]{neurips_2025}

\usepackage[table]{xcolor}
\usepackage[utf8]{inputenc} 
\usepackage[T1]{fontenc}    
\usepackage{hyperref}       
\usepackage{url}            
\usepackage{booktabs}       
\usepackage{amsfonts}       
\usepackage{nicefrac}       
\usepackage{microtype}      
\usepackage{xcolor}         
\usepackage{wrapfig}

\usepackage{amssymb} 
\usepackage{mathtools} 
\usepackage{mathrsfs}  
\usepackage{amsthm}
\usepackage{bbm}
\usepackage{bm}
\usepackage{booktabs}

\definecolor{tablegray}{gray}{0.95}

\theoremstyle{plain}
\newtheorem{theorem}{Theorem}[section]
\newtheorem{proposition}[theorem]{Proposition}

\theoremstyle{definition}
\newtheorem{definition}[theorem]{Definition}
\newtheorem{assumption}[theorem]{Assumption}
\theoremstyle{remark}

\usepackage{algorithm}
\usepackage{algpseudocode}

\newcommand{\mbf}[1]{{\mathbf{#1}}}

\usepackage{xparse} 
\NewDocumentCommand{\expectation}{O{\big} m m}{%
  \mathbb{E}_{#2}{#1[ #3 #1]}%
}

\newcommand{\Zeta}{Z}
\newcommand{\inv}{^{-1}}
\renewcommand{\mathcal}{\mathscr}

\usepackage[mathscr]{eucal}

\title{Reward-Aware Proto-Representations\\ in Reinforcement Learning}

\author{%
  Hon Tik Tse \quad Siddarth Chandrasekar \quad Marlos C. Machado* \\
  University of Alberta, Alberta Machine Intelligence Institute (Amii) \\
  *Canada CIFAR AI Chair\\
  \texttt{\{hontik, siddart2, machado\}@ualberta.ca}
}

\begin{document}

\maketitle

\begin{abstract}
  In recent years, the successor representation (SR) has attracted increasing attention in reinforcement learning (RL), and it has been used to address some of its key challenges, such as exploration, credit assignment, and generalization. The SR can be seen as representing the underlying credit assignment structure of the environment by implicitly encoding its induced transition dynamics. However, the SR is reward-agnostic. In this paper, we discuss a similar representation that also takes into account the reward dynamics of the problem. We study the default representation~(DR), a recently proposed representation with limited theoretical (and empirical) analysis. Here, we lay some of the theoretical foundation underlying the DR in the tabular case by (1) deriving dynamic programming and (2) temporal-difference methods to learn the DR, (3) characterizing the basis for the vector space of the DR, and (4) formally extending the DR to the function approximation case through default features. Empirically, we analyze the benefits of the DR in many of the settings in which the SR has been applied, including (1) reward shaping, (2) option discovery, (3) exploration, and (4) transfer learning. Our results show that, compared to the SR, the DR gives rise to qualitatively different, reward-aware behaviour and quantitatively better performance in several settings.
\end{abstract}

\section{Introduction}

Learning appropriate representations is a key challenge in reinforcement learning (RL). 
The successor representation (SR)~\citep{dayan1993improving}, which represents states as the expected discounted number of visits to successor states, is a particularly promising idea. It has been shown to be an effective distance metric for reward shaping~\citep{wulaplacian,wang2021towards}, a promising inductive bias for temporally-extended exploration~\citep{machado2019efficient,machado2023temporal}, and an effective representation for credit assignment~\citep{mahadevan2007proto, machado2018eigenoption}, generalization~\citep{le2022generalization}, and zero-shot RL~\citep{touati2021learning, TouatiRO23, BorsaBQMHMSS19, TirinzoniTFGKXL25, agarwal2025proto}.
Furthermore, recent evidence suggests that the SR is a good computational model for explaining decision-making~\citep{momennejad2017successor} and neural activity in the brain~\citep{stachenfeld2017hippocampus, stachenfeld2014design}. \looseness=-1

The underlying idea behind the SR is to incorporate the temporal aspect of the problem into the representation so that two states are considered similar if they lead to similar future outcomes. Similar ideas permeate other concepts such as proto-value functions~\citep{mahadevan2007proto, mahadevan2005proto} and slow feature analysis~\citep{sprekeler2011relation, wiskott2002slow}.
In this work, we use the term {\it proto-representations} to refer to such temporal representations that implicitly capture the environment dynamics, with the SR the most prominent of them.
However, the SR (and others~\citep{moskovitzfirst, moskovitz2023state, Blier2021learning, FarebrotherGALG23, WiltzerFGT0DBR24}) only captures the transition dynamics of the environment, and does not take rewards into account. 
In this paper, we demonstrate the benefits of learning a proto-representation that also takes rewards into account. \looseness=-1

One candidate for reward-aware proto-representations is the {\it default representation} (DR)~\citep{piray2021linear}. Piray and Daw, in a neuroscience study, applied the DR mainly to replanning tasks and explaining different cognitive phenomena, such as habits and cognitive control. 
Nevertheless, we claim that the DR is an interesting concept beyond neuroscience, due to its relationship with the SR and its reward-awareness. In this context, many aspects of the DR have yet to be explored.
There are no efficient, general algorithms for learning the DR incrementally and online with a linear cost, akin to temporal-difference (TD) learning~\citep{sutton1988learning}. 
Additionally, given the successful applications of the eigenvectors of the SR for reward shaping~\citep{wulaplacian, wang2021towards}, and temporally-extended exploration~\citep{machado2023temporal}, the eigenvectors of the DR have seen limited investigation.
Finally, the behavior of the DR for a reward function not constant over non-terminal states is underexplored. \looseness=-1

In this paper, we (1) derive methods for learning the DR using dynamic programming (DP) and temporal-difference (TD) learning~\citep{sutton1988learning}, and prove the convergence for the DP method, (2)~characterize the basis for the vector space of the DR, also describing the settings in which the DR and the SR have the same eigenvectors, and
(3) define default features, taking a first step towards using the DR with function approximation. \looseness=-1

Furthermore, we empirically investigate whether the DR 
provides benefits over the SR. While prior work~\citep{piray2021linear,bazarjani_piray_2025} has mainly focused on applying the DR to explain phenomena in neuroscience, here we consider applications more common in computational RL. 
We demonstrate that: (1) in environments with low-reward regions to be avoided, using the DR for reward shaping achieves superior performance over the SR, (2) using the DR for online eigenoption discovery~\citep{machado2023temporal}, the DR exhibits reward-aware exploratory behavior and obtains higher rewards than the SR over the course of exploration, (3) the DR, similar to the SR, can be used for count-based exploration, and (4) default features enable efficient transfer learning when the rewards at terminal states change.

\vspace{-0.1cm}
\section{Background}
\vspace{-0.1cm}
In this paper, we use lowercase symbols (e.g., $r, p$) to denote functions, calligraphic font (e.g, $\mathcal{S}, \mathcal{A}$) to denote sets, bold lowercase symbols (e.g., $\mbf r, \mbf e$) to denote vectors, and bold uppercase symbols (e.g., $\mbf \Psi, \mbf \Zeta$) to denote matrices. We index the $(i, j)$-th entry of a matrix $\mbf A$ by $\mbf A(i, j)$.

\vspace{-0.2cm}
\subsection{Reinforcement Learning and the Successor Representation}
\vspace{-0.1cm}
In the standard RL framework, the environment is formulated as an MDP $\langle\mathcal{S}, \mathcal{A}, r, p, \gamma\rangle$, where $\mathcal{S}$ is the state space, $\mathcal{A}$ is the action space, $r:\mathcal{S} \to \mathbb{R}$ (or $r:\mathcal{S} \times \mathcal{A} \to \mathbb{R}$) is the reward function, $p:\mathcal{S} \times \mathcal{A} \to \Delta (\mathcal{S} )$ is the transition function, and $\gamma \in [0, 1)$ is the discount factor. 
At every time step $t$, the agent observes a state, $S_t$, and selects an action, $A_t$. The environment then transitions to a next state, $S_{t + 1} \sim p(\cdot |S_t, A_t)$, and the agent receives a reward, $R_t = r(S_t)$ (or $R_t=r(S_t, A_t)$).
The agent's goal is to learn a policy $\pi:\mathcal{S} \to \Delta( \mathcal{A} )$ that maximizes the expected discounted return $\mathbb{E}_{\pi}[\sum_{t=0}^\infty \gamma^t R_t]$. \looseness=-1

In this framework, the successor representation~\citep[SR;][]{dayan1993improving} represents states as the expected discounted number of visits to their successor states. Given a policy $\pi$, the SR, $\mathbf{\Psi}^\pi \in \mathbb{R}^{|\mathcal{S}| \times |\mathcal{S}|}$, is defined as
\vspace{-0.1cm}
\begin{align}
    \mathbf{\Psi}^\pi(s, s') = \mathbb{E}_{\pi}\left[
    \sum_{t=0}^\infty \gamma^t \mathbbm{1}_{\{ S_t = s' \}} | S_0 = s
    \right], \label{eq:SR_definition}
\end{align}
where $\mathbbm{1}$ is the indicator function. Furthermore, denoting the transition probability matrix induced by $p$ and $\pi$ by $\mathbf{P}^{\pi}$, the SR can be computed in closed form by $\sum_{t=0}^\infty \gamma^t (\mbf P^\pi)^t = (\mathbf{I} - \gamma \mathbf{P}^{\pi})^{-1}$. {Note that in this work, we set the rows corresponding to terminal states to be all zeros for $\mathbf{P}^{\pi}$. Denoting the set of trajectories from $s$ to $s'$ by $\mathcal{T}_{s\to s'}$,} we can express the entries of the SR~\citep{Blier2021learning} as
\vspace{-0.1cm}
\begin{align}
    \mbf \Psi^\pi(s, s') = \sum_{{\tau \in \mathcal{T}_{s \to s'} }} \mbf P^\pi (\tau) \gamma^{\eta(\tau)}
    \label{eq:SR_expanded},
\end{align} 
where $\mbf P^\pi(\tau)$ is the probability of following $\tau$ under $\pi$, and $\eta(\tau)$ denotes the number of steps in $\tau$. 

Finally, the SR can be learned using temporal-difference learning~\citep{sutton1988learning, dayan1993improving}, for all $j \in \mathcal{S}$:
\begin{align}
    \mbf \Psi^\pi(S_t, j) \longleftarrow \mbf \Psi^\pi(S_t, j) + \alpha \big[
    \mathbbm{1}_{ \{S_t = j \}} + \gamma \mbf \Psi^\pi(S_{t + 1}, j) - \mbf \Psi^\pi(S_t, j)
    \big].
\end{align}

\vspace{-0.3cm}
\subsection{Linearly Solvable Markov Decision Processes}
\vspace{-0.1cm}
The {\it default representation} (DR) is defined in the framework of linearly solvable MDPs~\citep{todorov2006linearly, todorov2009efficient}, a simplification of MDPs in which the optimal value function can be expressed by a linear equation. 
In this work, building on the work by Piray and Daw \citep{piray2021linear}, we define linearly solvable MDPs as $(\mathcal{S}, \mathcal{A}, r, p, \pi_d)$, where $\mathcal{S}$ is the state space, $\mathcal{A}$ is the action space, $r:\mathcal{S} \to \mathbb{R}$ is the reward function, $p: \mathcal{S} \times \mathcal{A} \to {\Delta (\mathcal{S})} $ is the transition function, and $\pi_d$ is the {\it default policy} {that assigns non-zero probabilities to all state-action pairs.}
It is standard in linearly solvable MDPs to assume no discounting of rewards, i.e., $\gamma = 1$. 
{We also make the assumption that $r(s) < 0$ for all non-terminal $s$.}
In order to make MDPs linearly solvable, the agent is tasked not only with maximizing the rewards, but also with not deviating too much from the default policy. Deviating from the default policy incurs a cost to the agent, causing the agent to also receive a penalty $\lambda \operatorname{KL}(p^\pi (\cdot | S_t) \| p^{\pi_d}(\cdot | S_t))$ at every time step $t$, where $\operatorname{KL}(u\|v)$ denotes Kullback-Leibler (KL) divergence between $u$ and $v$, $p^\pi(s' | s)$ denotes the probability of transitioning to $s'$ from $s$ under policy $\pi$, and $\lambda > 0$ determines the relative importance of the deviation cost. \looseness=-1

The DR, introduced under this formulation, is a reward-aware proto-representation that encodes the expected rewards for visiting successor states. Let $\mathbf{r} \in \mathbb{R}^{|\mathcal{S}|}$ be the vector of rewards at all states, and $\mathbf{P}^{\pi_d}$ be the transition probability matrix induced by $\pi_d$ (i.e., $\mathbf{P}^{\pi_d}(s, s')$ is the probability of transition from $s$ to $s'$ under $\pi_d$), the DR, denoted by $\mathbf{\Zeta}$, can be computed in closed form by 
\begin{align}
    \mbf Z = \Bigl [\operatorname{diag} \bigl(\exp(\mathbf{-r / \lambda}) \bigr) - \mathbf{P}^{\pi_d} \Bigr]^{-1}, \label{eq:DR_definition}
\end{align}
where $\exp$ denotes element-wise exponentiation. 

Importantly, the DR can be used to retrieve the optimal value function with a set of linear equations \citep{piray2021linear}. Let $N, T$ be the set of indices of non-terminal and terminal states, and $\mathcal{S}_N$ and $\mathcal{S}_T$ denote the set of non-terminal and terminal states. We have \looseness=-1
\begin{align}
    \exp(\mbf v^*_N / \lambda) = \mbf \Zeta_{NN} \mbf P^{\pi_d}_{NT} \exp(\mbf r_T / \lambda),
    \label{eq:DR_optimal_v}
\end{align}
where $\mbf v^*_N$ is the vector of optimal values for non-terminal states, and $\mbf P^{\pi_d}_{NT} \in \mathbb{R}^{|\mathcal{S}_N| \times |\mathcal{S}_T|}$ denotes the matrix of transition probabilities. 
Note that after learning the DR, assuming access to the transition dynamics, it is possible to directly compute the optimal values for any new configurations of terminal state rewards, a fact that has been used in the past to perform transfer learning \citep{piray2021linear,bazarjani_piray_2025}.

\vspace{-0.2cm}
\section{Reward-Aware Proto-Representations}
\label{sec:reward-aware-proto-representations}
\vspace{-0.1cm}
In the previous section, we presented the DR in a much more limited way than when discussing the SR. This is not by accident; up to this point, the DR has not been that well developed. Little is written about the impact of environment rewards on the DR, how those impact the space the DR spans, and we still do not have a state-action formulation of the DR. We address these issues in this section.

\vspace{-0.1cm}
\subsection{The Default Representation as a Generalization of the Successor Representation} \label{sec:eigenvectors}
\vspace{-0.1cm}
\begin{wrapfigure}[22]{r}{0.35\textwidth}
    \centering
    \includegraphics[width=0.8\linewidth]{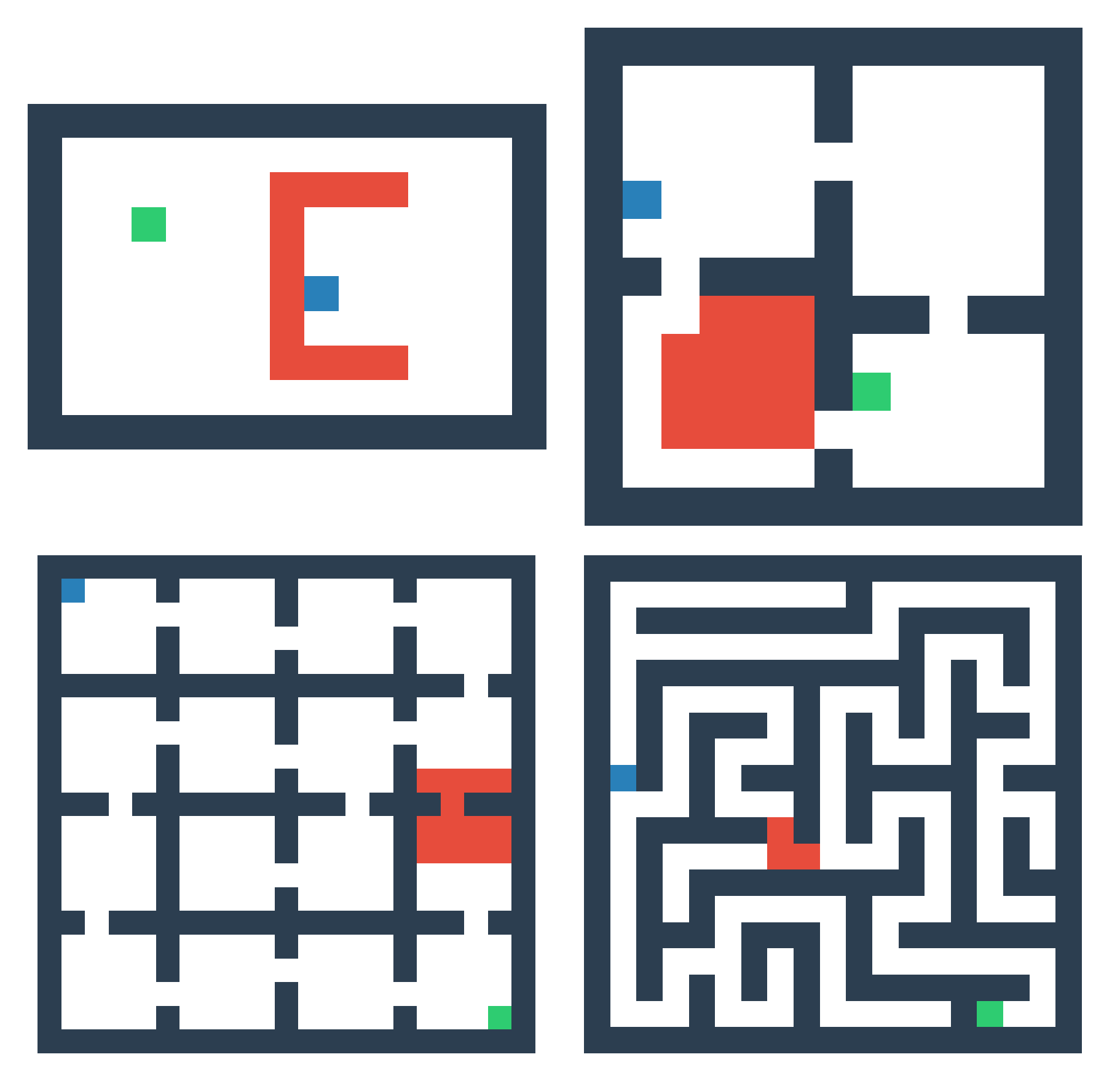}
    \caption{Episodic envs. adapted to incorporate negative rewards. Clockwise: 1)~\emph{grid task}~\citep{dayan1993improving}, 2)~{\it four rooms}~\citep{sutton1999between}, 3)~{\it grid room}~\citep{wang2021towards}, and 4)~{\it grid maze}~\citep{wang2021towards}. Start state is in blue. The agent receives $-1$ reward at every time step unless it steps on red tiles ($-20$ reward) or reaches the goal in green ($0$ reward).}
    \label{fig:environments}
\end{wrapfigure}
In problems in which the reward signal is the same across the whole state space, the DR spans the same space spanned by the SR. However, in settings in which the rewards received by the agent vary throughout an episode (see Figure~\ref{fig:environments}), the DR captures the expected rewards obtained when traveling between two states, instead of the number of transitions required to travel between states, as in the SR. \looseness=-1

We first formalize the idea that the DR and SR span the same space in Theorem~\ref{theorem:SR_DR_equiv_main} in settings in which the agent does not have access to a reward signal. Prior work had been limited to comparing the analytical form of the SR and DR~\citep{bazarjani_piray_2025}.

\begin{theorem}
    \label{theorem:SR_DR_equiv_main}
    Suppose both the SR and DR are computed with respect to the same policy, i.e., $\pi = \pi_d$. When the reward function is constant and negative, i.e., $r(s) = r(s') < 0 \ \forall s, s' \in \mathcal{S}$, the SR and DR share the same set of eigenvectors. Furthermore, when the SR and DR are symmetric, the $i$-th eigenvectors of the SR and DR are equivalent, and the $i$-th eigenvalues of the SR ($\mu_{\text{SR}, i}$) and DR ($\mu_{\text{DR}, i}$) are related as follows:
    \begin{align}
        \mu_{\text{SR}, i} = \bigg[\gamma \Big(\mu_{\text{DR}, i}^{-1} - \exp\big(-r(s)/\lambda \big) + \gamma^{-1} \Big)\bigg]^{-1},
    \end{align}
    {where $\gamma\in(0,1)$ is the discount factor of the SR, $r(s)$ is the state reward, and $\lambda$ is the relative importance of the deviation cost of the DR.}
\end{theorem}

\emph{Proof sketch.} We write the DR as $[\tilde r \mbf I - \mbf P^\pi]^{-1}$, where $\tilde r = \exp(-r(s) / \lambda)$, and perform algebraic manipulations on the expression $\mathbf{Z}\mathbf{e} = \lambda \mathbf{e}$ to obtain the eigenvectors of $[\mbf I - \gamma \mbf P^\pi]^{-1}$. Then, we show that the eigenvalues of the SR and DR are related by a function that is monotonically increasing over the range of DR eigenvalues. The complete proof is available in Appendix~\ref{appendix:theorem3.1_proof}.

This result is particularly interesting because the eigenvectors of the SR have been shown to be useful in settings such as reward shaping and option discovery. {The top eigenvector of the SR captures temporal distance between states, and such distance has been used for reward shaping as it leads to better performance than distances that are based on the coordinates of states ~\citep{wang2021towards,wulaplacian}. In option discovery, it has been shown that eigenoptions, which are obtained by training a policy to maximize a reward signal induced by the eigenvectors of the SR, allow the agent to reach less explored areas of the state space, facilitating exploration~\citep{machado2023temporal}. We will revisit these applications in Section~\ref{sec:experiments}.}

\begin{figure}
    \centering
    \includegraphics[width=\linewidth]{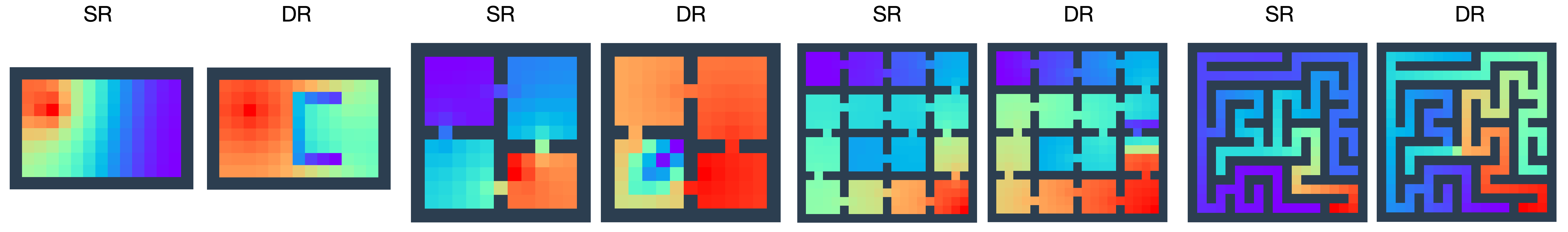}
    \vspace{-0.4cm}
    \caption{Top eigenvectors of the SR and the DR in the environments shown in Figure~\ref{fig:environments}. We report the logarithm of the DR for better visualization due to very different magnitudes.}
    \label{fig:eigenvectors}
\end{figure}

In problems in which the reward function is not constant, the DR is able to capture such information. While this is evident by the reward vector, $\mbf r$, in Equation~\ref{eq:DR_definition}, we can re-write the DR to match the closed-form of the SR~(Equation~\ref{eq:SR_expanded}) to make such a relationship clearer: 
\begin{align}
    \mbf \Zeta (s, s') = \sum_{{\tau \in \mathcal{T}_{s\to s'}}} \mbf P^{\pi_d}(\tau) \exp (r(\tau) / \lambda ),
    \label{eq:DR_expanded}
\end{align}
where $r(\tau)$ is the sum of rewards obtained along trajectory $\tau$ (see derivation in Appendix~\ref{appendix:lemma_derivation}). \looseness=-1

By writing Equation~\ref{eq:DR_expanded} similarly to the SR's, we see that both the SR and the DR compute a weighted average of a statistic over all trajectories from $s$ to $s'$, where the weights correspond to the probabilities of the trajectories under the policies w.r.t. which the SR and DR are computed. For the SR, the statistic captures the number of transitions required to travel between states, while for the DR, it captures the rewards obtained when travelling between states. \looseness=-1

Figure~\ref{fig:eigenvectors} illustrates the eigenvectors of the DR and SR in the environments depicted in Figure~\ref{fig:environments} to help the reader develop an intuition of the similarities and differences between these two proto-representations. While the top eigenvectors of the SR reflect the proximity of states in terms of transitions, the top eigenvectors of the DR not only capture the transition dynamics of the environment, but they also accurately reflect the positions of low-reward regions in the environment.

\subsection{State-Action-Dependent Rewards}
\label{section:state-action-dependent-rewards}
The form of the DR discussed so far is defined for reward functions that depend only on the state, and cannot be applied when the reward function depends on state-action pairs, which is prevalent in RL. 

To extend the DR to state-action pairs, we need to extend linearly solvable MDPs to state-action pairs, {which was first presented in the work by~\citet{ringstrom2020goal}. However, the DR was not discussed in this work, so we provide our derivation in Appendix~\ref{appendix:DR_state_action_pair}.} Interestingly, the resulting equation looks similar to Equation~\ref{eq:DR_definition}. Using the overbar for vectors and matrices in the state-action setting, e.g., $\bar {\mbf{ P}}$ instead of $\mbf P$, the DR for state-action-dependent rewards is: \looseness=-1
\begin{align}
    \bar{\mbf \Zeta} = \Big[\operatorname{diag}\big(\exp(-\bar{\mbf r} / \lambda)\big) - \bar{\mbf P}^{\pi_d} \Big]\inv,
    \label{eq:DR_definition_SA}
\end{align}
where and $\bar{\mbf r} \in \mathbb{R}^{|\mathcal{S}||\mathcal{A}|}$ is the vector of state-action rewards, and $\bar{\mbf P}^{\pi_d} \in \mathbb{R}^{|\mathcal{S}||\mathcal{A}| \times |\mathcal{S}||\mathcal{A}|}$ the matrix of transition probabilities between state-action pairs under $\pi_d$.

By learning the DR for the state-action space, the optimal Q-values can now be computed by
\begin{align}
\exp(\bar{\mbf q}_N / \lambda) = \bar{\mbf \Zeta}_{NN}  \bar{\mbf P}^{\pi_d}_{NT} \exp(\bar{\mbf r}_T/\lambda),
\end{align}
where $\bar{\mbf q} \in \mathbb{R}^{|\mathcal{S}||\mathcal{A}|}$ is the vector of optimal state-action values.
Furthermore, the optimal policy can be computed simply as
\begin{align}
    \pi^*(a|s) = \frac{\pi_d(a|s) \exp (q^*(s,a) / \lambda )}{\sum_{a'} \pi_d(a'|s) \exp (q^*(s,a') / \lambda )}.
    \label{eq:DR_optimal_policy}
\end{align}
As we show in Section~\ref{sec:default_features}, the DR under this formulation, combined with default features, unlike prior work, allows direct computation of optimal policies when terminal rewards change, enabling efficient transfer without assuming access to transition dynamics~\citep{piray2021linear,bazarjani_piray_2025}.

\section{General Algorithms for Estimating the Default Representation}
We now introduce efficient and general algorithms for learning the DR through dynamic programming and TD learning. Before the results we outline below, there were two update mechanisms for the DR: one based on the Woodbury matrix identity~\citep{piray2021linear}, which is limited to the tabular case; and a TD learning method recently introduced in a pre-print that is limited to settings in which the reward function is constant over non-terminal states~\citep{bazarjani_piray_2025}. We refer the reader back to the previous section to emphasize how different the DR is in more general settings, with rewards that vary across states. \looseness=-1

\vspace{-0.1cm}
\subsection{Dynamic Programming}
We first present a method to learn the DR by dynamic programming (DP), and we prove it converges. The update rule for the DP algorithm and the convergence result are presented in the theorem below.
\begin{theorem}
    \label{theorem:dp_convergence_main}
    Let $\mathcal{S}_N$ denote the set of non-terminal states. Assume $r(s) < 0 \ \forall s \in \mathcal{S}_N$. Let $\mbf R = \operatorname{diag} \big(\exp(-\mbf r / \lambda) \big)$, where $\mbf r$ is the vector of all state rewards. Let $\mbf \Zeta_0 = \mbf R^{-1}$. The update rule
    \begin{align}
        \mbf \Zeta_{k + 1} = \mbf R^{-1} + \mbf R^{-1} \mbf P^{\pi_d} \mbf \Zeta_k 
        \label{eq:DR_DP_update_2}
    \end{align}
    converges to the DR, that is, $\lim_{k \to \infty} \mbf \Zeta_{k} = \mbf \Zeta$.
\end{theorem}

\emph{Proof sketch.} We recursively expand the RHS of Equation~\ref{eq:DR_DP_update_2} and look at its closed-form in the limit, which we show is a Neumann series. We conclude the proof by leveraging the convergence of the Neumann series and plugging it back into the derived limit. The complete proof is in Appendix~\ref{appendix:DP_convergence_proof}. \looseness=-1

\vspace{-0.1cm}
\subsection{Temporal Difference Learning}
To learn the DR by TD learning, we build on the DP method and start with Equation~\ref{eq:DR_DP_update_2}. For any state $s$, for all $j$, $\mbf \Zeta(s, j)$ can be learned using DP as
\begin{align}
    \mbf Z(s, j) = \exp \big(r(s) / \lambda \big) \Big( \mathbbm{1}_{\{s = j\}} +  
    \mathbb{E}_{s' \sim p^{\pi_d}(s' | s)} \big[{\mbf \Zeta (s', j)} \big] \Big),
\end{align}
which requires access to the transition probabilities under the default policy. 

In TD learning, we do not assume access to these probabilities, but we can approximate the above equation using samples, estimating the expectation over next states $s'$ with a sampled next state.
Then, given a transition $(s, a, r, s')$ sampled under the default policy, we can update $\mbf \Zeta(s, j)$ as
\begin{align}
    \mbf Z(s, j) \longleftarrow \mbf \Zeta(s, j) + \alpha \Big [ Y
    - \mbf \Zeta(s, j) \Big], \label{eq:DR_TD}
\end{align}
where $\alpha$ is the step size and $Y$ the bootstrapping target. For non-terminal states, 
$Y = \exp(r / \lambda) \big (\mathbbm{1}_{\{ s = j \}} + \mbf Z(s', j) \big)$, while for terminal states, Y = $\exp(r / \lambda) \mathbbm{1}_{\{s = j\}}$. Naturally, we can leverage importance sampling~\citep{precup2000eligibility} if the default policy is different from the behaviour policy. \looseness=-1

\vspace{-0.1cm}
\section{Extending The DR Beyond Classical Tabular RL}~\label{sec:extensions}

\vspace{-0.6cm}

We have been using the tabular case to formalize the ideas we are introducing; now we extend them beyond this setting. Given Equation~\ref{eq:DR_TD}, it is trivial to formalize the parameterization of the DR such that $\mbf Z(s, s'; \boldsymbol{\theta}) \approx \mbf Z(s, s')$, for some parameters $\boldsymbol{\theta}$. In this section, we go beyond that, defining the concept of \emph{default features}, which is more semantically meaningful when talking about the DR in terms of features. {We also define a proto-representation for the maximum entropy RL framework. However, since this paper focuses on the DR, we defer the discussion of this proto-representation to Appendix~\ref{appendix:mer}.} \looseness=-1

\vspace{-0.1cm}
\label{sec:default_features}

Similar to successor features~\citep[SFs;][]{barreto2017successor}, we derive here a decomposition of the value function into features and weights such that these features capture the dynamics of the environment. Different from SFs, default features also consider non-terminal rewards as part of the environment dynamics. \looseness=-1

Recall we can use the DR to obtain the optimal value function with a set of linear equations (Eq.~\ref{eq:DR_optimal_v}):
\begin{align}
    \exp(\mbf v^*_N / \lambda) = \mbf \Zeta_{NN} \mbf P^{\pi_d}_{NT} \exp(\mbf r_T / \lambda).
\end{align}
Inspired by the decomposition performed with SFs, we define default features (DFs) by decomposing the terminal rewards in this equation.
We assume that the exponential of the terminal state rewards can be computed as $\exp \big(r(s) / \lambda \big) = \boldsymbol{\phi}(s)^\top \mbf w$, where $\boldsymbol{\phi}(s), \mbf w \in \mathbb{R}^d$ are the features of a terminal state, $s$, and the weights of the reward function, respectively. Note that this decomposition is not restrictive, since we can fully recover any terminal state rewards if $\boldsymbol\phi(s) = \exp \big(r(s) / \lambda \big)$.

Under this decomposition, the above equation can be written as
\begin{align}
    \exp(\mbf v^*_N / \lambda) = \mbf \Zeta_{NN} \mbf P^{\pi_d}_{NT} \exp(\mbf r_T / \lambda) = \mbf \Zeta_{NN} \mbf P^{\pi_d}_{NT} \mbf \Phi \mbf w,
\end{align}
where $\mbf \Phi \in \mathbb{R}^{|S_T| \times d}$ is the matrix of terminal state features.
We define the product $\mbf \Zeta_{NN} \mbf P^{\pi_d}_{NT} \mbf \Phi$ as the DFs matrix, and each row $\boldsymbol \zeta(s) =  (\mbf \Zeta_{NN} \mbf P^{\pi_d}_{NT} \mbf \Phi)_{s,:}$ to be the DFs of a non-terminal state.
Then, the optimal value for each non-terminal state can be computed as $\boldsymbol \zeta(s)^\top\mbf w$. {Note that such a form of optimal value computation is only possible when only the rewards at terminal states change. The reason is that, when only the terminal state rewards change, $\zeta(s)$ stays the same for all non-terminal $s$. However, when the rewards for non-terminal states change, $\zeta(s)$ changes since $\zeta(s)=\mathbf{Z}_{NN}\mathbf{P}^{\pi_d}_{NT} \mathbf{\Phi}$ and $\mathbf{Z}_{NN}$ depends on non-terminal state rewards.}

Building on Theorem~\ref{theorem:dp_convergence_main}, we can learn the DFs using TD learning from 
a transition $(s, a, r, s')$ sampled under the default policy:
\begin{align}
    \boldsymbol\zeta(s) \longleftarrow \boldsymbol\zeta(s) + \alpha \big ( \exp(r / \lambda) \boldsymbol\zeta(s') - \boldsymbol\zeta(s) \big),
\end{align}
where $\alpha$ is the step size, and we define $\boldsymbol\zeta(s) \coloneqq \boldsymbol\phi(s)$ for terminal states, $s$. DFs can be thought of as features propagating from the terminal states to all non-terminal states in a manner that respects the reward and transition dynamics. The definition of DFs can be easily extended to the case of state-action-dependent rewards, following the formulation in Section~\ref{section:state-action-dependent-rewards}.

Apart from enabling function approximation, another benefit of learning the DFs over the DR is that it does not require access to the transition dynamics, $\mbf P^{\pi_d}_{NT}$, to retrieve optimal values. This is especially useful when learning the DFs for state-action-dependent rewards, since we can directly retrieve the optimal policy using the corresponding state-action values (see Equation~\ref{eq:DR_optimal_policy}).
After learning the DFs in this setting, even without access to environment dynamics, we can efficiently compute the optimal policy under any terminal state reward configuration, enabling transfer learning, as shown in Section~\ref{sec:experiments}. {Note that while DFs enable directly computing the optimal policy, they are limited to the scenario where only the terminal state rewards change. On the other hand, while SFs can only compute a policy as good as the ones used to compute the SFs, they can be applied when any state rewards change, allowing more flexibility than DFs.}
\looseness=-1

\vspace{-0.2cm}

\vspace{-0.1cm}
\section{Experiments}~\label{sec:experiments}
\vspace{-0.6cm}

Proto-representations, like the SR, have been used in RL for reward shaping~\citep{wulaplacian, wang2021towards}, option discovery~\citep{machado2017laplacian, machado2023temporal}, count-based exploration~\citep{machado2020count}, and transfer~\citep{barreto2017successor}, among others. We now revisit these settings to assess the impact of using \emph{reward-aware} representations. \looseness=-1

\vspace{-0.2cm}

\subsection{Reward Shaping}

Drawing inspiration from results with the SR~\citep{wulaplacian, wang2021towards}, we first evaluate the effectiveness of the DR as a distance metric. In reward shaping experiments, our goal is to assess whether the DR can serve as a useful shaping signal, capturing the geometry of the state space, and whether it not only guides the agent toward the goal state but also helps it avoid negative rewards.

Assuming $\mbf e$ is the top eigenvector of the SR, prior work~\citep{wulaplacian, wang2021towards} has used the shaping reward $\hat r_t = - \big(\mbf e(s_\text{goal}) - \mbf e(s_{t+1}) \big)^2$, where $\mbf e(s)$ denotes the entry for state $s$. Here, we instead use the top eigenvector $\mbf e$ of the DR for potential-based reward shaping~\citep{ng1999policy}, defining the shaping reward as $\hat r_t = \gamma \mbf e(s_{t+1}) - \mbf e(s_t)$.
We compare our approach (\texttt{DR-pot}) with three baselines: 1)~potential-based shaping with the SR's top eigenvector (\texttt{SR-pot}), 2)~the distance-based SR method from~\citep{wang2021towards} (\texttt{SR-pri}), and 3)~no shaping (\texttt{ns}). {Both the SR and DR are computed with respect to the uniform random policy.} \looseness=-1

We carry out the experiments in the environments shown in Figure~\ref{fig:environments}. Importantly, the shortest path from the start state to the goal state passes through low-reward regions and is sub-optimal, so the agent needs to learn to avoid low-reward regions in the environment to achieve the optimal return. 
For the reward shaping approaches, we train a Q-learning~\citep{watkins1989learning} agent using a convex combination of the original environment reward, $r_t$, and the shaping reward, $\hat r_t$, resulting in the expression $ (1 - \beta) r_t + \beta \hat r_t$, where $\beta \in [0, 1]$ is a hyperparameter. {Note that we assume access to the eigenvectors of the SR and DR prior to training the agent with the potential-based reward. Future work can explore learning the eigenvectors and the policy simultaneously.} For the no shaping baseline, we simply train the Q-learning agent using the original environment reward. We use $\gamma = 0.99$, $\epsilon = 0.05$ for $\epsilon$-greedy exploration, $\lambda=1.3$ for the DR, and perform a grid search over the Q-learning's step size ($[0.1,0.3,1.0]$) and $\beta$ ($[0.25, 0.5, 0.75, 1.0]$). We run 20 seeds for each hyperparameter setting, and after identifying the best hyperparameters, re-run 50 seeds to avoid maximization bias.

\begin{figure}
    \centering
    \includegraphics[width=\linewidth]{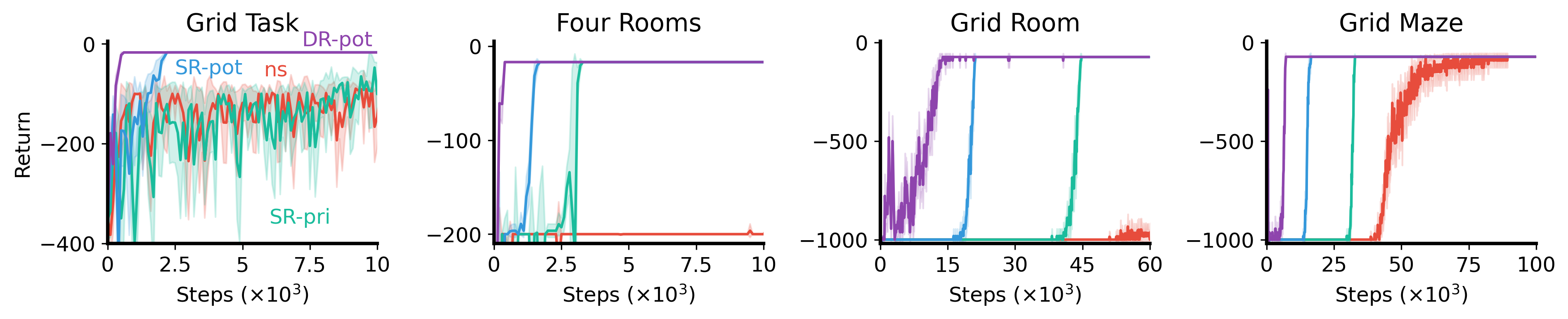}
    \vspace{-0.4cm}
    \caption{The avg. undiscounted return over 50 runs for potential-based reward shaping using the DR (\texttt{DR-pot}), the SR (\texttt{SR-pot}), the prior approach using the SR (\texttt{SR-prior})~\citep{wang2021towards}, and no shaping (\texttt{ns}) in the environments shown in Figure~\ref{fig:environments}. The shaded area indicates 95\% confidence interval. \looseness=-1}
    \label{fig:reward_shaping}
\end{figure}

Figure~\ref{fig:reward_shaping} shows the undiscounted returns of our approach and the baselines in the environments shown in Figure~\ref{fig:environments}. First, we observe that compared to prior methods, potential-based reward shaping is a better way of utilizing the eigenvectors of the proto-representations. 
Second, we observe that \texttt{DR-pot} performs significantly better than any of the baselines, including the SR-based approaches. This is because, as shown in Figure~\ref{fig:eigenvectors}, the top eigenvector of the DR captures the location of low-reward regions, and is capable of guiding the agent along the optimal path that does not cross them, whereas the top eigenvector of the SR fails to do so, and simply leads the agent along the shortest sub-optimal path to the goal. As per our result in Section~\ref{sec:eigenvectors}, the performance using the SR and the DR is very similar 
in the absence of low-reward regions
(see Appendix~\ref{appendix:reward_shaping}). \looseness=-1

\vspace{-0.2cm}
\subsection{Option Discovery}
\vspace{-0.1cm}

The options framework~\citep{precup2000temporal, sutton1999between} allows the agent to interact with the environment in a temporally-extended manner. Prior work~\citep{machado2023temporal} demonstrated that the SR can be used for discovering options. In particular, their iterative online eigenoption discovery method using the SR, called {\it covering eigenoptions (CEO)}, greatly reduces the average number of steps required to visit every state, and enables exploring the state space more uniformly. Furthermore, CEO combined with deep learning has shown promising results and outperformed state-of-the-art baselines in a wide variety of settings~\citep{klissarov2023deep}. Here, we consider using the DR for iterative online eigenoption discovery. \looseness=-1

Our approach, {\it reward-aware covering eigenoptions (RACE)}, is similar to CEO but learns the DR instead of the SR, using its top eigenvector to define the intrinsic reward. 
{In RACE, at every iteration, the agent collects samples to learn the DR, the eigenvector of which is then used to compute an eigenoption. The eigenoption will be used in the next RACE iteration to collect samples. This forms a cycle of using eigenoptions to improve the learned representation, and using the learned representation to then refine eigenoptions.}
The full algorithm and hyperparameters are in Appendix~\ref{appendix:eigenoption_discovery}. We compare RACE with CEO and a uniform random walk in the environments from Figure~\ref{fig:environments}, using 100-step episodes with no terminal states. To evaluate exploration, we measure the percentage of states visited; to assess risk awareness, we track the average reward per time step. Our goal is to test whether RACE enables reward-aware exploration. The results are depicted in Figure~\ref{fig:eigoption_discovery_scatter}, where each point represents the average performance of one hyperparameter setting over 10 seeds. 

\begin{figure}
    \centering
    \includegraphics[width=\linewidth]{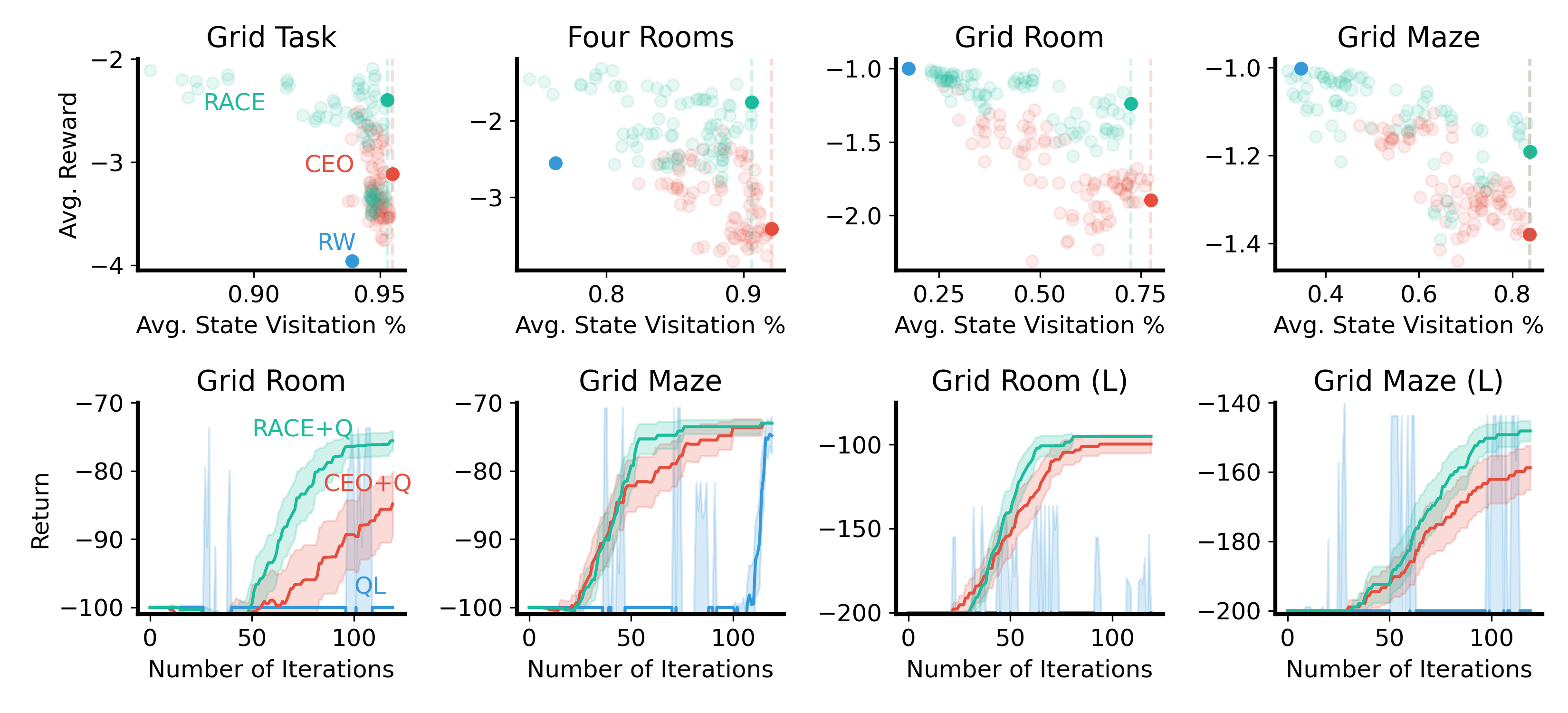}

    \vspace{-0.3cm}
    \caption{\emph{Top row}: Average reward vs. state visitation percentage for various hyperparameter settings of iterative online eigenoption discovery via CEO, RACE, and random walk (RW) in the environments from Figure~\ref{fig:environments}. For reference, solid dots mark settings with the highest visitation.
    \emph{Bottom row}: Undiscounted return of RACE+Q, CEO+Q, and QL (baseline), averaged over 50 seeds. Rightmost environments are shown in Figure~\ref{fig:env_larger}. Shaded areas indicate 95\% confidence intervals.}
    \label{fig:eigoption_discovery_scatter}
\end{figure}

RACE is better than CEO at avoiding low-reward regions. Interestingly, because of its reward-awareness, RACE has a slightly lower average state visitation percentage than CEO. The agent sometimes needs to take detours to visit states that are obstructed by low-reward regions, thus requiring a larger number of steps to visit other states (see Figure~\ref{fig:eigenoption_discovery_cumulative_visit}, in Appendix~\ref{appendix:eigenoption_discovery}). 
Again, a consequence of our Theorem~\ref{theorem:SR_DR_equiv_main} is that CEO and RACE behave very similarly in environments with constant rewards across the whole state space (see Figure~\ref{fig:eigenoption_discovery_without_lava_curve}, in Appendix~\ref{appendix:eigenoption_discovery}). Finally, RACE's exploratory behaviour also leads to faster learning when accumulating rewards. Results combining iterative online eigenoption discovery with Q-learning~\citep{watkins1989learning} are shown in Figure~\ref{fig:eigoption_discovery_scatter}. Appendix~\ref{appendix:eigenoption_discovery} provides details about these experiments. Interestingly, RACE+Q performs much better than CEO+Q because CEO does not distinguish between high-reward and low-reward paths, and can reach the goal state following a sub-optimal one, causing offline Q-learning to first learn the sub-optimal path.

\vspace{-0.2cm}
\subsection{Count-Based Exploration}
\vspace{-0.1cm}

\begin{wraptable}{r}{0.6\textwidth} 
\vspace{-0.6cm}
\centering
\caption{Model-free count-based exploration. Values in thousands ($\times 10^3$). 95\% conf. intervals shown in parentheses.}
\vspace{0.3cm}
\label{table:count_based_results}
\small
\rowcolors{2}{white}{tablegray}
\begin{tabular}{@{}lccc@{}}
\toprule
\textbf{Environment} & \textbf{Sarsa} & \textbf{+SR} & \textbf{+DR} \\
\midrule
\textsc{RiverSwim} & 25 (0.8) & 1,206 (566) & 2,964 (252) \\
\textsc{SixArms}   & 265 (157) & 1,066 (2,708) & 3,518 (4,571) \\
\bottomrule
\end{tabular}
\end{wraptable}

Machado et al.~\citep{machado2020count} have shown that the norm of the SR, while it is being learned, encodes state-visitation counts. While reward-aware representations are somewhat at odds with pure exploration, given their risk-averse nature, we can still ask whether the norm of the DR can be used as a density model for pseudocounts~\citep{bellemare2016unifying}. Here, we evaluate the use of the $\ell_2$ norm of the DR as an exploration bonus. Specifically, the intrinsic reward is defined as $r_{\text{intr}}(s, a) = \beta \cdot \log( \| \bar{\mbf \Zeta}_{sa,:} \|_2)$,
where $\beta$ is a scaling factor. \looseness=-1

We follow Machado et al.'s~\citep{machado2020count} experimental setup, and we consider two traditional exploration problems: Riverswim and SixArms.  In Table~\ref{table:count_based_results}, we report the performance (total undiscounted return) of Sarsa~\citep{rummery1994online} with $\epsilon$-greedy exploration, Sarsa + SR~\citep{machado2020count}, and Sarsa + DR. Details about hyperameters are available in Appendix~\ref{appendix:count-based-exploration}. All results are averaged over 100 independent runs.

Our results suggest that the norm of the DR can indeed work as a density function since its use has led to performance orders of magnitude better than random exploration. Maybe surprisingly, Sarsa+DR even outperforms, with statistical significance, Sarsa+SR in RiverSwim. This is an interesting result, as, to learn the optimal policy, the agent needs to overcome its aversion to negative rewards, which it is encoding in its representation; but maybe this is precisely why it is somewhat more effective, as such a representation allows it to learn faster. This interplay is an interesting topic of future work. \looseness=-1

\subsection{Transfer Learning}

The DR has been applied in the transfer learning setting to compute the optimal policy when the rewards at terminal states change~\citep{bazarjani_piray_2025,piray2021linear}. However, such computations require access to the transition dynamics of the environment. With our extension of the DR to state-action-dependent rewards and default features, it is now possible to compute the optimal policy without requiring access to the transition dynamics. 
We compare our approach with successor features~\citep[SFs;][]{barreto2017successor}, which allow efficient transfer when the rewards of any states change. Note, however, that in this work we only consider the setting when only the rewards at terminal states change, since this is the setting that DFs can be applied to. The comparison of DFs and SFs under this setting does not serve as evidence that SFs is not effective for the setting it is designed for. \looseness=-1

\begin{wrapfigure}{r}{0.5\textwidth}
    \centering
    \includegraphics[width=0.48\textwidth]{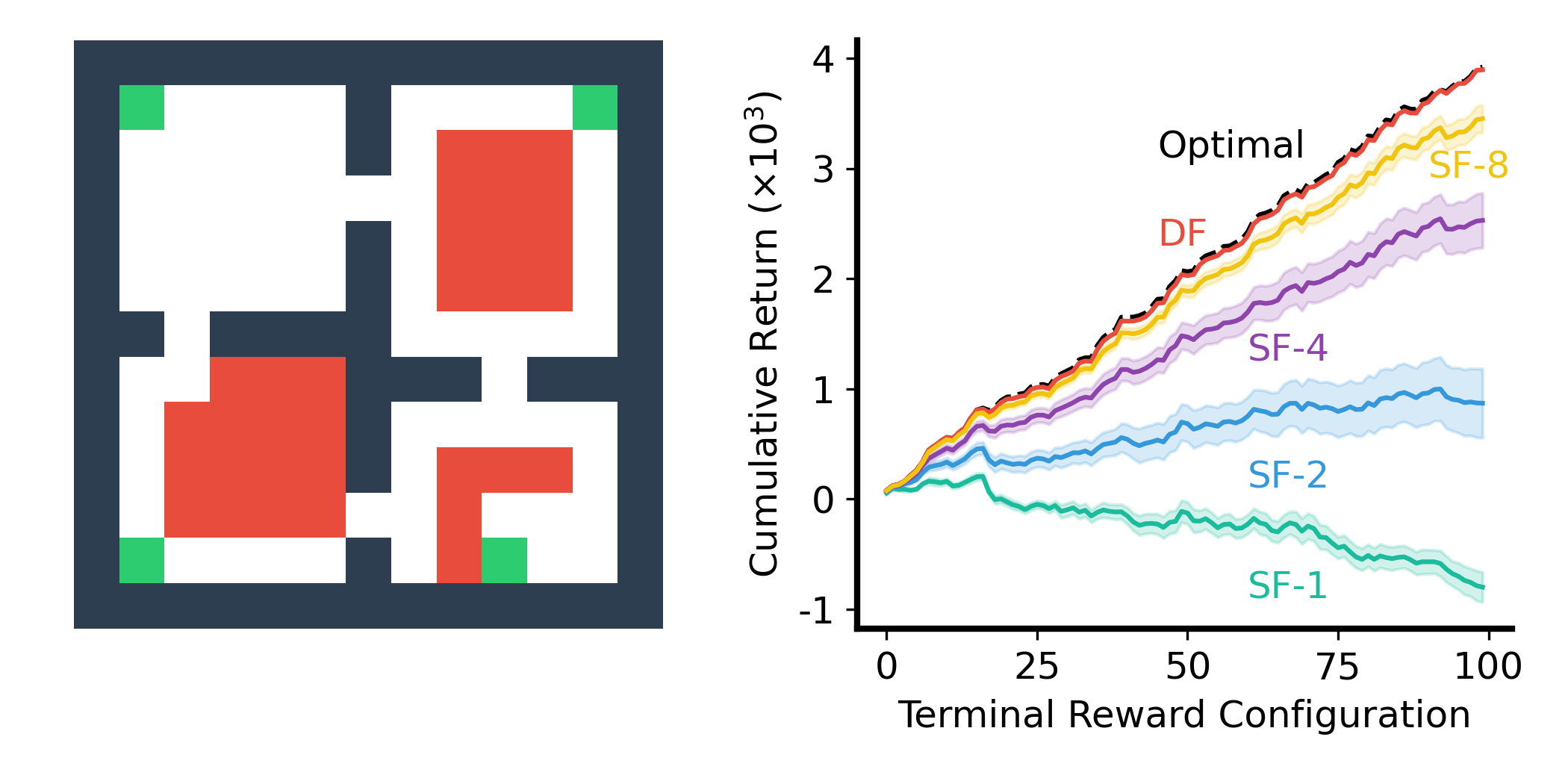}
    \caption{Left: Four rooms with multiple goals. Right: Cumulative return across new terminal state reward configurations
    . Curves are averaged over 50 runs. The shaded area shows 95\% conf. interval. \looseness=-1}
    \label{fig:transfer}
\end{wrapfigure}

We perform experiments in the environment shown in Figure~\ref{fig:transfer}, consisting of four goal states. We describe our procedure for learning and using the SFs and DFs for transfer in Appendix~\ref{appendix:transfer}, as well as the features $\boldsymbol{\phi}$ used. Figure~\ref{fig:transfer} shows the performance of the DFs, the SFs computed with respect to different numbers of policies, and the optimal performance obtained by training an optimal policy under each new terminal state reward configuration using Q-learning~\citep{watkins1989learning}. As previously mentioned in Section~\ref{sec:default_features}, in this setting, DFs is extremely effective because it directly computes the optimal policy for new terminal state reward configurations, while SFs only computes a policy that is at least as good as the policies used to compute the SFs. Note, however, that DFs computes an optimal policy in the linearly solvable MDP setting that balances the reward function and the cost of deviating from the default policy. \looseness=-1

\vspace{-0.2cm}
\section{Related Work}
\vspace{-0.1cm}

Throughout this paper, we have discussed the relevant related work for both the theoretical development and the empirical evaluation of our work. A line of work we did not discuss that might appear relevant for learning representations that capture both reward and transition dynamics is based on bisimulation metrics~\citep{ferns2004metrics, castro2020scalable}, which quantify behavioural similarity between pairs of states by comparing their immediate rewards and distributions over the following states. 
These works enforce the bisimulation metric structure on the representation space by encouraging the difference between state representations to approximate the bisimulation metric~\citep{zhang2021learning,castro2021mico, zang2022simsr, kemertas2021towards}. However, such approaches learn a distance-preserving mapping, where behaviorally similar states are mapped to nearby representations. This is in contrast to the DR or the SR, which encodes the environment dynamics directly in a semantically meaningful manner (see Eq.~\ref{eq:DR_expanded} and~\ref{eq:SR_expanded}).
As an immediate result, the DR supports planning (see Eq.~\ref{eq:DR_optimal_v}) as well as other use cases discussed here, such as exploration and option discovery, in a manner not possible with bisimulation-metric-based representations. \looseness=-1

\vspace{-0.2cm}
\section{Conclusion}
\vspace{-0.1cm}

In this work, we advance the theoretical foundation of the default representation (DR), a reward-aware proto-representation. Our contributions include: (1) deriving dynamic programming and TD methods to learn the DR, (2) analyzing its eigenspectrum, and (3) extending it to function approximation via default features. Empirically, we show that the DR, by incorporating reward structure, yields qualitatively different behavior than the SR across tasks such as reward shaping, option discovery, exploration, and transfer.

{While the DR provides benefits in environments where being reward-aware is important, there are some limitations. First, naively computing the DR and its eigendecomposition can lead to numerical instabilities due to taking the exponential of negative rewards, as discussed in Appendix~\ref{appendix:numerical_considerations}. 
This could be problematic, especially in more complicated environments with long horizons. 
Second, as we show in Figures~\ref{fig:reward_shaping_no_low_reward} and~\ref{fig:eigenoption_discovery_without_lava_curve} in Appendix~\ref{appendix:exp-detail}, in environments in which all non-terminal state rewards are the same, the DR will perform as well as the SR. Given the numerical issues associated with the DR, the SR may be a better choice in these environments. Third, while the DR captures the environment dynamics more fully, this comes at the expense of flexibility. While the SR, which is reward-agnostic, can allow transfer learning when any state rewards change, the DR, which is reward-aware, can only perform transfer in settings where only the terminal state rewards change.
}

Finally, this work lays a foundation for the DR as a potential middle ground between reward-agnostic and risk-sensitive approaches.
While we focused on the tabular case for clarity and to be able to evaluate the DR across a range of use cases where reward-agnostic proto-representations have been applied, restricting the scope to this setting remains a limitation of this work. Future work should explore the use of the DR in more complex settings, starting with adapting neural network methods developed for approximating the SR~\citep[e.g.,][]{kulkarni2016deep,gomez2024proper,chua2024learning}. \looseness=-1

\begin{ack}
We thank Brett Daley for useful discussions and feedback. We thank anonymous reviewers whose feedback improved this work.
This research was supported in part by the Natural Sciences and Engineering Research Council of Canada (NSERC), the Canada CIFAR AI Chair Program, and Alberta Innovates. It was also enabled in part by computational resources provided by the Digital Research Alliance of Canada.
\end{ack}

\medskip


\bibliography{bib}

\medskip

\clearpage

\newpage
\appendix

\section{Broader Impact}
\label{appendix:broader_impact}
This work investigates a novel form of representation in reinforcement learning---a foundational area within the field of artificial intelligence. Given the fundamental and theoretical nature of this research, it does not pose any foreseeable negative societal impacts. As with most projects focused on core scientific inquiry in AI, the associated risks are minimal, and no significant mitigation strategies are required. \looseness=-1

\section{Theory}
\label{appendix:theory}
In this section, we present the detailed derivation of the theoretical results in the main text.

\subsection{Derivation of Equations~\ref{eq:SR_expanded} and~\ref{eq:DR_expanded}}
\label{appendix:lemma_derivation}
We present the full derivation of Equations~\ref{eq:SR_expanded} and~\ref{eq:DR_expanded}.

\subsubsection{{Derivation of Equation~\ref{eq:DR_expanded}}}
Equation~\ref{eq:DR_expanded} states that the $(s_i, s_j)$-th entry of the DR can be written as:
\begin{align}
    \mbf \Zeta (s_i, s_j) = \sum_{\tau \in \mathcal{T}_{s_i\to s_j}} \mbf P^{\pi_d}(\tau) \exp\big(r(\tau) / \lambda\big).
\end{align}

\begin{proof}
Recall that the analytical form of the DR~\citep{piray2021linear} is $\mbf \Zeta =\Big[\operatorname{diag} \big(\exp(\mathbf{-r / \lambda})\big) - \mathbf{P}^{\pi_d}\Big]^{-1}$. To simplify notation, let $\mathbf{R} = \operatorname{diag} \big(\exp(-\mathbf{r}/\lambda) \big)$. 
We then have
\begin{align}
    \mbf \Zeta  &= (\mbf R - \mbf P^{\pi_d})\inv \\
    (\mbf R - \mbf P^{\pi_d}) \mbf Z &= \mbf I \\
    \mbf R \mbf Z - \mbf P^{\pi_d} \mbf Z &= \mbf I \\
    \mbf R \mbf Z &= \mbf I + \mbf P^{\pi_d} \mbf \Zeta \\
    \mbf Z &= \mbf R\inv + \mbf R\inv \mbf P^{\pi_d} \mbf \Zeta.
\end{align}

From the above equation, we know that
\begin{align}
    \mathbf{\Zeta}(s_i, s_j) = \exp \big(r(s_i)/\lambda \big) \mathbbm{1}_{\{ s_i=s_j \}} + \sum_{s'} \mathbf{P}^{\pi_d}(s_i, s') \exp \big(r(s_i)/\lambda \big)  \mathbf{\Zeta}(s', s_j),
\end{align}
the right-hand side of which can be repeatedly expanded to give
\begin{align}
    \mbf \Zeta(s_i, s_j) =& \exp \big(r(s_i) / \lambda \big) \mathbbm{1}_{ \{s_i = s_j\}} \nonumber \\
    &+ \sum_{s'} \mbf P^{\pi_d}(s_i, s') \exp \Big ( \big(r(s_i) + r(s') \big) / \lambda \Big) \mathbbm{1}_{ \{s' = s_j\}} \nonumber \\
    &+ \sum_{s'} \mbf P^{\pi_d}(s_i, s') \sum_{s''} \mbf P^{\pi_d}(s', s'') \exp \Big ( \big(r(s_i) + r(s') + r(s'') \big) / \lambda \Big) \mathbbm{1}_{ \{s'' = s_j\}} \nonumber \\
    &+ \dots.
\end{align}

In the above equation, it can be seen that the expression enumerates over trajectories from state $s_i$ to state $s_j$ with different lengths. The first row corresponds to a single trajectory with length 1 (containing one state), which is the simple case when the agent starts and ends in $s_j$ without taking an action. The second row enumerates all trajectories from $s_i$ to $s_j$ with length 2, and so on. 

The above equation can be written in the following more compact form:
\begin{align}
    \mbf \Zeta(s_i, s_j) = \sum_{\tau \in \mathcal{T}_{s_i\to s_j}} \mbf P^{\pi_d}(\tau) \exp \big (r(\tau) / \lambda \big), 
\end{align}
where $\mathcal{T}_{s_i \to s_j}$ is the set of trajectories from state $s_i$ to state $s_j$, $\mbf P^{\pi_d}(\tau)$ is the probability of following trajectory $\tau$ under the policy $\pi_d$, and $r(\tau)$ is the sum of rewards obtained in trajectory $\tau$. \looseness=-1
\end{proof}

\subsubsection{{Derivation of Equation~\ref{eq:SR_expanded}}}
{Equation~\ref{eq:SR_expanded} appeared in the work by~\citet{Blier2021learning}, but the derivation was not shown. For completeness, we provide the derivation here, which is similar to that for Equation~\ref{eq:DR_expanded}.}
Equation~\ref{eq:SR_expanded} states that the $(s_i, s_j)$-th entry of the SR can be written as: \looseness=-1
\begin{align}
    \mbf \Psi^\pi(s_i, s_j) = \sum_{\tau \in \mathcal{T}_{s_i\to s_j}} \mbf P^\pi (\tau) \gamma^{\eta(\tau)}.
\end{align}

\begin{proof}
Recall that the analytical form of the SR~\citep{dayan1993improving} is $\mbf \Psi^\pi = (\mbf I - \gamma \mbf P^\pi)\inv$. 
We know that
\begin{align}
    \mbf \Psi^\pi &= (\mbf I - \gamma \mbf P^\pi)\inv \\
     (\mbf I - \gamma \mbf P^\pi) \mbf \Psi^\pi &= \mbf I \\
    \mbf \Psi^\pi &= \mbf I + \gamma \mbf P^\pi \mbf \Psi^\pi,
\end{align}
from which we know that 
\begin{align}
    \mbf \Psi^\pi(s_i, s_j) = \mathbbm{1}_{\{ s_i = s_j\}} + \sum_{s'} \mbf P^\pi(s_i, s') \gamma \mbf \Psi^\pi(s', s_j).
\end{align}
Expanding the right-hand side, we have
\begin{align}
    \mbf \Psi^\pi(s_i, s_j) =& \mathbbm{1}_{\{ s_i = s_j\}} \\
    &+ \sum_{s'} \mbf P^\pi(s_i, s') \gamma \mathbbm{1}_{\{ s' = s_j\}} \\
    &+ \sum_{s'} \mbf P^\pi(s_i, s') \sum_{s''} \mbf P^\pi(s', s'') \gamma^2 \mathbbm{1}_{\{ s'' = s_j\}} \\
    &+ \dots,
\end{align}
which can be written as
\begin{align}
    \mbf \Psi^\pi(s_i, s_j) = \sum_{\tau \in \mathcal{T}_{s_i \to s_j}} \mbf P^\pi(\tau) \gamma^{\eta(\tau)}, \label{eq:SR_expected_view}
\end{align}
where $\mathcal{T}_{s_i \to s_j}$ is the set of trajectories from state $s_i$ to state $s_j$, $\mbf P^{\pi}(\tau)$ is the probability of following trajectory $\tau$ under the policy $\pi$, and $\eta(\tau)$ is the number of steps it takes to visit $s_j$ from $s_i$ in trajectory $\tau$. Note that when $\tau = (s_i) = (s_j)$, $\eta(\tau) = 0$.
\end{proof}

\subsection{{Theorem 3.1}}
\label{appendix:theorem3.1_proof}
\renewcommand{\thetheorem}{3.1}
\begin{theorem}
    Suppose both the SR and DR are computed with respect to the same policy, i.e., $\pi = \pi_d$. When the reward function is constant and negative, i.e., $r(s) = r(s') < 0 \ \forall s, s' \in \mathcal{S}$, the SR and DR share the same set of eigenvectors. Furthermore, when the SR and DR are symmetric, the $i$-th eigenvectors of the SR and DR are equivalent, and the $i$-th eigenvalues of the SR ($\mu_{\text{SR}, i}$) and DR ($\mu_{\text{DR}, i}$) are related as follows:
    \begin{align}
        \mu_{\text{SR}, i} = \bigg[\gamma \Big(\mu_{\text{DR}, i}^{-1} - \exp\big(-r(s)/\lambda \big) + \gamma^{-1} \Big)\bigg]^{-1},
    \end{align}
    where $\gamma\in(0,1)$ is the discount factor of the SR, $r(s)$ is the state reward, and $\lambda$ is the relative importance of the deviation cost of the DR.
\end{theorem}

\begin{proof}
    Let $\mu_i, \mbf e_i$ be the $i$-th eigenvalue and eigenvector of the DR. Assuming that the reward function is constant, and letting $\tilde r = \exp \big (-r(s) / \lambda \big)$, the DR can be written as $[\tilde r \mbf I - \mbf P^\pi]^{-1}$. Then, we have:
    \begin{align}
        [\tilde r \mbf I - \mbf P^{\pi}]^{-1} \mbf e_i &= \mu_i \mbf e_i \\
        [\tilde r \mbf I - \mbf P^{\pi}] \mbf e_i &= \mu_i^{-1} \mbf e_i \\
        - \gamma \mbf P^{\pi} \mbf e_i &= \gamma (\mu_i^{-1} - \tilde{r}) \mbf e_i \\
        (\mbf I - \gamma \mbf P^{\pi}) \mbf e_i &= \gamma (\mu_i^{-1} - \tilde{r} + \gamma ^{-1}) \mbf e_i \\
        (\mbf I - \gamma \mbf P^{\pi})^{-1} \mbf e_i &= \big[\gamma (\mu_i^{-1} - \tilde{r} + \gamma ^{-1})\big]^{-1} \mbf e_i \\
        \mbf \Psi^\pi \mbf e_i &= \bigg[\gamma \Big(\mu_i\inv - \exp\big(-r(s) / \lambda\big) + \gamma\inv \Big)\bigg]\inv \mbf e_i \label{eq:42}
    \end{align}

    We have thus shown that the SR and DR share the same set of eigenvectors. When the SR and DR are symmetric, their eigenvalues are real, and we can further show that the orders of the eigenvectors of the SR and DR by their corresponding eigenvalues are identical.

    From Eq.~\ref{eq:42}, we know that the eigenvalues of the SR and DR are related by the following function:
    \begin{align}
        \mu_{SR} = f(\mu_{DR}) = [\gamma (\mu_{DR}^{-1} - \exp(-r(s) / \lambda) + \gamma \inv)]\inv
    \end{align}

    Taking the derivative, we have
    \begin{align}
        f'(\mu_{DR}) = {\gamma\inv [1 + \mu_{DR} (\gamma \inv - \exp(-r(s) / \lambda)) ]^{-2}}.
    \end{align}
    When $\gamma = \exp(r(s) / \lambda)$, $f'(\mu_{DR})$ is always positive, so the orders of the eigenvectors of the SR and DR are identical. When $\gamma \neq \exp(r(s) / \lambda)$, $f'(\mu_{DR}) > 0$ except at the vertical asymptote at $\mu_{DR} = \frac{-1}{\gamma\inv - \exp(-r(s) / \lambda)}$. We now show that all eigenvalues of the DR lie on the same side of the vertical asymptote, and so the orders of eigenvectors are still identical. We do so by first deriving lower and upper bounds on the eigenvalues of the DR.

    First, we show that all eigenvalues of the DR are positive. We start with the inverse of the DR, $\mathbf{Z}\inv= \operatorname{diag}(\exp(-\mathbf{r} / \lambda)) - \mathbf{P}^{\pi_d}$. Consider a row $i$ of $\mathbf{Z}\inv$ corresponding to a non-terminal state. We know that $\mathbf{Z}\inv(i, i) = \exp(-r(s) / \lambda) - \mathbf{P}^{\pi_d}(i, i) > 1 - \mathbf{P}^{\pi_d}(i, i) = \sum_{j\neq i}|\mathbf{Z}\inv(i, j)|$. For a row $i$ corresponding to a terminal state, $\mathbf{Z}\inv(i, i) > \sum_{j\neq i}|\mathbf{Z}\inv(i, j)|$ easily holds. Then, by the Gershgorin circle theorem, we know that this matrix has all positive eigenvalues, and thus, all eigenvalues of the DR are positive. 

    Next, we derive an upper bound on the eigenvalues of the DR. Also applying the Gershgorin circle theorem on the inverse of the DR, we know that the set of eigenvalues of the DR lies within the union of the Gershgorin discs. For the $i$-th row of the inverse of the DR, if it corresponds to a non-terminal state, we have a disc centered at $\exp(-r(s) / \lambda) - \mathbf{P}^{\pi_d}(i, i)$ with radius $1 - \mathbf{P}^{\pi_d}(i, i)$. For a terminal state, the disc is centered at $\exp(-r(s) / \lambda)$ with radius 0. It is then not hard to see that the minimum of the union of the discs is simply $\exp(-r(s) / \lambda) - 1$. Since the lower bound on the eigenvalues of the inverse of the DR is $\exp(-r(s) / \lambda) - 1$, the upper bound on the eigenvalues of the DR is $\frac{1}{\exp(-r(s) / \lambda) - 1}$. 

    We now know that all eigenvalues of the DR lie in the range $\left(0, \frac{1}{\exp(-r(s) / \lambda) - 1} \right]$. We now show that the vertical asymptote of $f$, if it exists, always lies outside of this range. We focus on the following two cases: 1) $\gamma < \exp(r(s) / \lambda)$, and 2) $\gamma > \exp(r(s) / \lambda)$. 

    First, when $\gamma < \exp(r(s) / \lambda)$, it can be easily shown that the vertical asymptote $\frac{-1}{\gamma\inv - \exp(-r(s) / \lambda)} < 0$. Since all eigenvalues are positive, they all lie on the right hand side of the vertical asymptote.

    For the remaining case, we first have that
    \begin{align}
        \gamma &< 1 \\
        \gamma \inv  &> 1 \\
        \gamma\inv - \exp(-r(s) / \lambda) &> 1 - \exp(-r(s) / \lambda).
    \end{align}
    Since $\gamma > \exp(r(s) / \lambda)$ and we assume $r(s) < 0$, both sides of the inequality are negative. Then,
    \begin{align}
        \frac{1}{\gamma \inv - \exp(-r(s) / \lambda)} &< \frac{1}{1 - \exp(-r(s) / \lambda)} \\
        \frac{-1}{\gamma\inv - \exp(-r(s) / \lambda)} &> \frac{1}{\exp(-r(s) / \lambda) - 1}.
    \end{align}
    Since all eigenvalues are less than or equal to $\frac{1}{\exp(-r(s) / \lambda) - 1}$, they all lie on the left hand side of the vertical asymptote. We have thus shown that when a vertical asymptote exists, it lies outside of the range of eigenvalues, and so the orders of the eigenvectors of the SR and DR are still identical. 
\end{proof}

\subsection{Extension of the DR to State-Action-Dependent Reward}
\label{appendix:DR_state_action_pair}
In this subsection, we derive the DR for the case when the reward function depends on both the state and action. {The extension of linearly solvable MDPs to the state-action case was first presented by ~\citet{ringstrom2020goal}. However, since (1) the DR was not defined in this work, and (2) there are minor differences in the formulation, derivation, and notation, we present the full derivation below nonetheless. Note, also, that the derivation bears similarity to relative entropy policy search~\citep{peters2010relative}, which constrains the KL divergence between an observed data distribution and that generated by the policy.} \looseness=-1

For this setting, instead of formulating the deviation cost as $\operatorname{KL}\big(p^\pi(\cdot | S_t) \| p^{\pi_d}(\cdot | S_t) \big)$, we formulate the deviation cost from the default policy as $\operatorname{KL} \big(p^\pi(\cdot, \cdot | S_t, A_t) \| p^{\pi_d} (\cdot,\cdot | S_t, A_t) \big)$, where $p^\pi(\cdot, \cdot | S_t, A_t)$ denotes the distribution over the next state-action pairs given the current state-action pair $(S_t, A_t)$ and the policy $\pi$. Then, at every time step $t$, the agent receives a reward of $\tilde r(S_t, A_t) = r(S_t, A_t) -\lambda \operatorname{KL}\big(p^\pi(\cdot, \cdot | S_t, A_t) \| p^{\pi_d} (\cdot,\cdot | S_t, A_t)\big)$, where $\lambda > 0$ determines the relative importance of the deviation cost.

Let $T$ be the random variable denoting the final time step of an episode. The return starting from time $t$ is then 
\begin{align}
    G_t = \tilde r(S_t, A_t) + \dots + \tilde r(S_{T-1},A_{T-1}) + r(S_T, A_T),
\end{align}
where we allow the agent to select an action $A_t$ at the terminal state $S_t$ to obtain a final reward $r(S_T, A_T)$. \looseness=-1

The state-action value function is defined as $q^\pi(s, a) = \expectation{\pi}{G_t | S_t=s, A_t = a}$. For terminal states, $q^\pi(s, a)$ is simply $r(s, a)$. For non-terminal states, we have
\begin{align}
    q^\pi(s, a) &= \expectation{\pi}{G_t | S_t = s, A_t = a} \\
    &= \expectation[\big]{\pi}{\tilde r(S_t, A_t) + \dots + \tilde r(S_{T-1},A_{T-1}) + r(S_T, A_T) | S_t=s, A_t=a} \\
    &= \tilde r(s, a) + \expectation[\big]{\pi}{\tilde r(S_{t+1}, A_{t+1}) + \dots + \tilde r(S_{T-1},A_{T-1}) + r(S_T, A_T) | S_t=s, A_t = a} \\
    &= \tilde r(s, a) + \sum_{s', a'} p^\pi(s', a' | s, a) \mathbb{E}_\pi \big[\tilde r(S_{t+1}, A_{t+1}) + \dots + \\
    & \quad \quad \tilde r(S_{T-1},A_{T-1}) + r(S_T, A_T) | S_t=s, A_t = a, S_{t+1}=s', A_{t + 1}=a'\big] \\
    &= \tilde r(s, a) + \sum_{s', a'} p^\pi(s', a' | s, a) \mathbb{E}_\pi \big[\tilde r(S_{t+1}, A_{t+1}) + \dots + \\
    & \quad \quad \tilde r(S_{T-1},A_{T-1}) + r(S_T, A_T) |S_{t+1}=s', A_{t + 1}=a'\big] \\
    &= \tilde r(s, a) + \sum_{s', a'} p^\pi(s', a' | s, a) q^\pi(s', a') \\
    &= \tilde r(s, a) + \expectation[\big]{s', a' \sim p^\pi(\cdot, \cdot | s, a)}{q^\pi(s',a')}.
\end{align}

Then, the optimal state-action value function for non-terminal states satisfies
\begin{align}
    q^*(s,a) &= \max_\pi \Big \{
    \tilde r(s, a) + \expectation[\big]{s', a' \sim p^\pi(\cdot, \cdot | s, a)}{q^*(s',a')}
    \Big\} \\
    &= \max_\pi \Big\{
    r(s, a) -\lambda \operatorname{KL}\big(p^\pi(\cdot,\cdot|s, a) \| p^{\pi_d}(\cdot, \cdot | s, a)\big) + \expectation[\big]{s', a' \sim p^\pi(\cdot, \cdot | s, a)}{q^*(s',a')}
    \Big\} \\
    &= r(s, a) + \max_\pi \left \{
    -\lambda \expectation[\bigg]{s', a' \sim p^\pi(\cdot, \cdot |s, a)}{\ln\frac{p^\pi(s', a' | s, a)}{p^{\pi_d}(s', a' |s,a)} - \frac{q^*(s',a')}{\lambda}}
    \right \} \\
    &= r(s,a) + \max_\pi \left \{
    -\lambda \expectation[\bigg]{s', a' \sim p^\pi(\cdot, \cdot |s, a)} {\ln\frac{p^\pi(s', a' | s, a)}{p^{\pi_d}(s', a' |s,a)\exp\big(q^*(s',a')/\lambda\big)}}
    \right \}.
\end{align}
The expectation term above resembles a KL divergence, but the denominator of the fraction is not properly normalized. Define the normalization factor $C = \sum_{s', a'} p^{\pi_d}(s', a' | s,a) \exp\big(q^*(s',a')/\lambda\big)$. The above equation then becomes
\begin{align}
    q^*(s,a) &= r(s,a) + \max_\pi \left \{
    -\lambda \expectation[\bigg]{s', a' \sim p^\pi(\cdot, \cdot |s, a)} {\ln\frac{p^\pi(s', a' | s, a) / C}{p^{\pi_d}(s', a' |s,a)\exp\big(q^*(s',a')/\lambda\big) / C}}
    \right \} \\
    &= r(s,a) + \lambda \ln C + \max_\pi \bigg \{
    -\lambda \operatorname{KL} \Big(p^\pi(\cdot, \cdot|s, a)\|p^{\pi_d}(\cdot,\cdot|s,a)\exp \big(q^*(\cdot,\cdot)/\lambda \big)/C \Big)
    \bigg \} \\
    &= r(s,a)  + \lambda \ln C, 
    \label{eq:67}
\end{align}
where the last equality follows from the fact that the minimum of the KL divergence is 0. Note that in order to make the KL divergence 0, the optimal policy needs to satisfy 
\begin{align}
    p^{\pi^*}(s', a' | s,a) &\propto p^{\pi_d}(s', a'|s, a) \exp\big(q^*(s',a') / \lambda \big) \\
    \pi^*(a'|s') p(s'|s,a) &\propto \pi_d(a'|s')p(s'|s,a) \exp\big(q^*(s',a')/\lambda\big) \\
    \pi^*(a'|s')  &\propto \pi_d(a'|s') \exp \big(q^*(s',a')/\lambda \big)
\end{align}
Then, given the optimal state-action value function and the default policy, it is straightforward to retrieve the optimal policy, $\pi^*$:
\begin{align}
    \pi^*(a | s) = \frac{\pi_d(a | s) \exp \big(q^*(s,a) / \lambda \big)}{\sum_{a'} \pi_d(a'|s)\exp \big(q^*(s,a') / \lambda \big)}
\end{align}

Substituting $C$ back to Equation~\ref{eq:67}, we have
\begin{align}
    q^*(s,a) &= r(s,a) + \lambda \ln \sum_{s', a'} p^{\pi_d}(s', a' | s,a) \exp \big(q^*(s',a')/\lambda \big) \\
    q^*(s,a) / \lambda &= r(s,a) / \lambda + \ln \sum_{s', a'} p^{\pi_d}(s', a' | s,a) \exp \big(q^*(s',a')/\lambda \big) \\
    \exp \big(q^*(s,a)/\lambda \big) &= \exp \big(r(s,a)/\lambda \big) \left(\sum_{s', a'} p^{\pi_d}(s', a' | s,a) \exp\big(q^*(s',a')/\lambda\big) \right)
\end{align}

We now represent the above equation in matrix form. Let $\bar{\mbf q} \in \mathbb{R}^{|\mathcal{S}||\mathcal{A}|}$ be the vector of optimal state-action values, $\bar{\mbf r} \in \mathbb{R}^{|\mathcal{S}||\mathcal{A}|}$ be the vector of state-action rewards, and $\bar{\mbf P}^{\pi_d} \in \mathbb{R}^{|\mathcal{S}||\mathcal{A}| \times |\mathcal{S}||\mathcal{A}|}$ be the matrix of transition probabilities between state-action pairs under the default policy, i.e., $\bar{\mbf P}^{\pi_d}(sa,s'a') = p^{\pi_d}(s', a' | s, a)$. Furthermore, let $N, T$ be the set of indices of non-terminal and terminal state-action pairs respectively. The above equation, defined for non-terminal states, can then be written as
\begin{align}
    \exp(\bar{\mbf q}_N / \lambda) = \operatorname{diag}(\exp(\bar{\mbf r}_N / \lambda)) (\bar{\mbf P}^{\pi_d}_{NN} \exp(\bar{\mbf q}_N/\lambda) + \bar{\mbf P}^{\pi_d}_{NT}  \exp(\bar{\mbf q}_T/\lambda)) \\
    \operatorname{diag}(\exp(-\bar{\mbf r}_N / \lambda)) \exp(\bar{\mbf q}_N / \lambda) =  \bar{\mbf P}^{\pi_d}_{NN} \exp(\bar{\mbf q}_N/\lambda) + \bar{\mbf P}^{\pi_d}_{NT}  \exp(\bar{\mbf r}_T/\lambda)\\
    \big[\operatorname{diag}(\exp(-\bar{\mbf r}_N / \lambda)) - \bar{\mbf P}^{\pi_d}_{NN} \big] \exp(\bar{\mbf q}_N / \lambda) =  \bar{\mbf P}^{\pi_d}_{NT} \exp(\bar{\mbf r}_T/\lambda)  \\
     \exp(\bar{\mbf q}_N / \lambda) = \big[\operatorname{diag}(\exp(-\bar{\mbf r}_N / \lambda) )- \bar{\mbf P}^{\pi_d}_{NN} \big]\inv  \bar{\mbf P}^{\pi_d}_{NT} \exp(\bar{\mbf r}_T/\lambda)
\end{align}
The DR for state-action-dependent rewards is then $\big[\operatorname{diag}(\exp(-\bar{\mbf r}_N / \lambda)) - \bar{\mbf P}^{\pi_d}_{NN} \big]\inv$ for non-terminal states, and $\big[\operatorname{diag}(\exp(-\bar{\mbf r} / \lambda)) - \bar{\mbf P}^{\pi_d} \big]\inv$ for all states.

\subsection{Convergence Proof for Dynamic Programming Algorithm}
\label{appendix:DP_convergence_proof}
\renewcommand{\thetheorem}{4.1}
\begin{theorem}
    Let $\mathcal{S}_N$ denote the set of non-terminal states. Assume $r(s) < 0 \ \forall s \in \mathcal{S}_N$. Let $\mbf R = \operatorname{diag}(\exp(-\mbf r / \lambda))$, where $\mbf r$ is the vector of all state rewards. Let $\mbf \Zeta_0 = \mbf R^{-1}$. The update rule
    \begin{align}
        \mbf \Zeta_{k + 1} = \mbf R^{-1} + \mbf R^{-1} \mbf P^{\pi_d} \mbf \Zeta_k 
        \label{eq:DR_DP_update}
    \end{align}
    converges to the DR, that is, $\lim_{k \to \infty} \mbf \Zeta_{k} = \mbf \Zeta$.
\end{theorem}

\begin{proof}
    Recursively expanding the right hand side of Equation~\ref{eq:DR_DP_update}, we have
    \begin{align}
        \mbf \Zeta_{k + 1} &= \mbf R\inv + (\mbf R\inv \mbf P^{\pi_d}) \mbf R\inv + \dots + (\mbf R\inv \mbf P^{\pi_d})^{k + 1} \mbf R \inv \\
        &= 
        \left [ \sum_{t=0}^{k+1} (\mbf R\inv \mbf P^{\pi_d})^t \right ]
        \mbf R\inv .
    \end{align}

    Taking the limit, we have
    \begin{align}
        \lim_{k \to \infty} \mbf \Zeta_{k} &= \left [ \sum_{t=0}^{\infty} (\mbf R\inv \mbf P^{\pi_d})^t \right ] \mbf R\inv .
    \end{align}

    Note that $\left [ \sum_{t=0}^{\infty} (\mbf R\inv \mbf P^{\pi_d})^t \right ]$ is a Neumann series. We now establish the convergence of this Neumann series. Since $\sum_{s'}\mbf P^{\pi_d} (s, s')$ is equal to 1 for every non-terminal $s$, and 0 for all terminal $s$, we have \looseness=-1
    \begin{align}
        \| \mbf R^{-1} \mbf P^{\pi_d} \|_\infty = \max_{s \in \mathcal{S}} \exp \big(r(s) / \lambda \big) \sum_{s'} \mbf P^{\pi_d}(s, s') = \max_{s \in \mathcal{S}_N} \exp \big(r(s) / \lambda \big) < 1, \label{eq:Neumann_converge}
    \end{align}
    where the final inequality follows from our assumption that $r(s) < 0$. Equation~\ref{eq:Neumann_converge} is a sufficient condition for the convergence of the Neumann series, so we know that this series converges to $(\mbf I - \mbf R^{-1} \mbf P^{\pi_d})^{-1}$. Plugging this back, we have
    \begin{align}
        \lim_{k\to\infty}\mbf \Zeta_k 
        &= (\mbf I - \mbf R^{-1} \mbf P^{\pi_d})^{-1} \mbf R^{-1} \\
        &= \big[\mbf R^{-1} (\mbf R - \mbf P^{\pi_d})\big]^{-1} \mbf R^{-1} \\
        &= \big[\mbf R \mbf R^{-1} (\mbf R - \mbf P^{\pi_d})\big]^{-1} \quad\quad\quad \text{because $\mbf B\inv \mbf A\inv = (\mbf A\mbf B)\inv$} \\
        &= \Big[\operatorname{diag}\big(\exp(-\mbf r / \lambda)\big) - \mbf P^{\pi_d}\Big]^{-1} = \mbf \Zeta.
    \end{align}
\end{proof}

\subsection{{Proto-Representation in Maximum Entropy RL}}
\label{appendix:mer}
Similar to how the SR is a proto-representation derived in the standard RL formulation, and the DR in the linearly solvable MDPs, we can derive proto-representations for other formulations. Specifically, due to its popularity, and for completeness, we introduce a proto-representation in the maximum entropy (MaxEnt) RL framework~\citep{haarnoja2018soft, haarnoja2017reinforcement}. We call it {\it maximum entropy representation (MER)}. \looseness=-1

In our MaxEnt RL formulation, at each time step $t$, the agent receives both the reward $r(S_t)$ and an entropy bonus $\lambda \mathcal{H} \big(p^\pi(\cdot | S_t) \big)$, where $p^\pi(\cdot | S_t)$ denotes transition probabilities over successor states given policy $\pi$, and $\lambda > 0$ controls the weight of the entropy term. Assuming $\gamma = 1$, we define the MER as: \looseness=-1

\begin{definition}
    Let $\mbf r$ be the vector of state rewards, and $\mbf A$ be the adjacency matrix, i.e., $A(s, s')$ is equal to 1 if $s'$ can be reached in one step from $s$, and 0 otherwise. The MER is defined as
\begin{align}
    \mbf M = \Big[\operatorname{diag} \big(\exp(-\mbf r) / \lambda \big) - \mbf A \Big]\inv.
\end{align}
\end{definition}

The MER differs from the DR (Equation~\ref{eq:DR_definition}) in that it uses the adjacency matrix rather than the default policy’s transition probabilities. This connection arises because maximum entropy RL and linearly solvable MDPs are closely related as one can be transformed into the other~\citep{dvijotham2010inverse}. As a result, most of our contributions readily extend to the MER. Preliminary analyses showed that the MER also behaves similarly to the DR, so we focus our experiments on the DR. Still, the MER and DR are distinct proto-representations with different formulations, assumptions, and value functions, which can lead to different optimal policies. While we did not observe substantial differences in the settings we considered, exploring them in other regimes remains an open direction.

We now provide the derivation of the MER.
To derive the MER for deterministic transitions, we can simply formulate the entropy as $\mathcal{H}\big(\pi(\cdot | S_t)\big)$. However, to derive the MER for stochastic transitions, we need to formulate the entropy slightly differently.
We assume that at every time step $t$, the agent receives a reward of $\tilde r(S_t) = r(S_t) + \lambda \mathcal{H}\big(p^\pi(\cdot | S_t)\big)$, where $p^\pi(\cdot | S_t)$ denotes the transition probabilities over the successor states of $S_t$ given policy $\pi$, and $\lambda > 0$ determines the relative importance of the entropy term.  \looseness=-1

Let $T$ be a random variable denoting the final time step of an episode. The return starting from time $t$ is then
\begin{align}
    G_t = \tilde r(S_t) + \dots + \tilde r(S_{T - 1}) + r(S_T).
\end{align}
Note that similar to prior work~\citep{todorov2009efficient, piray2021linear}, we allow the agent to obtain a final reward of $r(S_T)$ at the terminal state, $S_T$.

The value function is defined as $v^\pi(s) = \expectation{\pi}{G_t | S_t = s}$. For terminal states, $v^\pi(s)$ is simply $r(s)$. For non-terminal states, we have
\begin{align}
    v^\pi(s) &= \expectation{\pi}{G_t | S_t = s} \\
    &= \expectation[\big]{\pi}{\tilde r(S_t) + \dots + \tilde r(S_{T - 1}) + r(S_T) | S_t = s} \\
    &= \tilde r(s) + \sum_{s'} p^\pi(s' | s) \expectation[\big]{\pi}{\tilde r(S_{t + 1}) + \dots + \tilde r(S_{T - 1}) + r(S_T) | S_t = s, S_{t + 1}=s'} \\
    &= \tilde r(s) + \sum_{s'} p^\pi(s' | s) \expectation[\big]{\pi}{\tilde r(S_{t + 1}) + \dots + \tilde r(S_{T - 1}) + r(S_T) | S_{t + 1}=s'} \\
    &= \tilde r(s) + \sum_{s'} p^\pi(s' | s) v^\pi(s') \\
    &= \tilde r(s) + \expectation[\big]{s' \sim p^\pi(\cdot | s)}{v^\pi(s')}.
\end{align}

The optimal value function for non-terminal states satisfies
\begin{align}
    v^*(s) &= \max_\pi \Big\{ \tilde r(s) + \expectation[\big]{s' \sim p^\pi(\cdot | s)}{v^*(s')} \Big\} \\
    &= \max_\pi \Big\{
    r(s) + \lambda \mathcal{H}\big(p^\pi(\cdot | s)\big) + \expectation[\big]{s' \sim p^\pi(\cdot | s)}{v^*(s')}
    \Big\} \\
    &= r(s) + \max_\pi \Big\{
    - \lambda \expectation[\big]{s' \sim p^\pi(\cdot | s)}{\ln p^\pi(s' | s)} + \expectation[\big]{s' \sim p^\pi(\cdot | s)}{v^*(s')} \Big\} \\
    &= r(s) + \max_\pi  \left \{ 
    -\lambda \expectation[\bigg]{s' \sim p^\pi(\cdot | s)}{\ln \frac{p^\pi(s'|s)}{\exp(v^*(s') / \lambda)}} \right \}.
\end{align}
The expectation term above resembles a KL divergence, but $\exp \big(v^*(s') / \lambda \big)$ is not a normalized distribution. To normalize it, we define $C(s) = \sum_{s' \in \mathcal{S}_s} \exp\big(v^*(s') / \lambda \big)$, where $\mathcal{S}_s = \big \{s' | \exists\pi: p^\pi(s' | s) > 0  \big\}$ denotes the set of successor states of state $s$. We then have
\begin{align}
    v^*(s) &= r(s) + \max_\pi \left \{ 
    -\lambda \expectation[\bigg]{s' \sim p^\pi(\cdot | s)}{\ln \frac{p^\pi(s'|s) / C(s)}{\exp\big(v^*(s') / \lambda\big) / C(s)}} \right \} \\
    &= r(s) + \max_\pi \left \{ -\lambda \expectation[\bigg]{s' \sim p^\pi(\cdot | s) }{\ln \frac{p^\pi(s'|s)}{\exp\big(v^*(s') / \lambda\big) / C(s)}} + \lambda \ln C(s) \right \} \\
    &= r(s) + \lambda \ln C(s) + \max_\pi \bigg \{ -\lambda
    \operatorname{KL} \Big(p^\pi(\cdot | s) \| \exp \big(v^*(\cdot) / \lambda \big)/C(s)\Big)
    \bigg\} \\
    &= r(s) + \lambda \ln C(s),
\end{align}
where the final equality follows from the fact that the minimum value of KL divergence is 0.

Then, we have that
\begin{align}
    v^*(s) &= r(s) + \lambda \ln \sum_{s' \in \mathcal{S}_s} \exp\big(v^*(s') / \lambda\big) \\
    v^*(s) / \lambda &= r(s) / \lambda + \ln \sum_{s' \in \mathcal{S}_s} \exp\big(v^*(s') / \lambda\big) \\
    \exp\big(v^*(s) / \lambda\big) &= \exp\big(r(s) / \lambda\big) \left ( \sum_{s' \in \mathcal{S}_s} \exp\big(v^*(s') / \lambda\big) \right )
\end{align}

Now, we express the above equation in a matrix form. Let $\mbf v$ be the vector of optimal state values, $\mbf r$ be the vector of state rewards, and $\mbf A$ be adjacency matrix, i.e. $A(s,s') = \mathbbm{1}_{\{ s' \in \mathcal{S}_{s} \}}$. Furthermore, let $N, T$ be the set of indices of non-terminal and terminal states respectively. The above equation, defined for non-terminal states, can be written as
\begin{align}
    \exp(\mbf v_N / \lambda) = \operatorname{diag}(\exp(\mbf r_N) / \lambda) \big( \mbf A_{NN} \exp(\mbf v_N / \lambda) + \mbf A_{NT} \exp(\mbf v_T / \lambda) \big) \\
     \operatorname{diag}\big(\exp(-\mbf r_N) / \lambda\big) \exp(\mbf v_N / \lambda) =  \mbf A_{NN} \exp(\mbf v_N / \lambda) + \mbf A_{NT} \exp(\mbf r_T / \lambda) \\
    \Big[\operatorname{diag}\big(\exp(-\mbf r_N) / \lambda\big) - \mbf A_{NN} \Big]\exp(\mbf v_N / \lambda) =    \mbf A_{NT} \exp(\mbf r_T / \lambda) \\
    \exp(\mbf v_N / \lambda) =  \Big[\operatorname{diag}\big(\exp(-\mbf r_N) / \lambda\big) - \mbf A_{NN} \Big]\inv  \mbf A_{NT} \exp(\mbf r_T / \lambda)
\end{align}

We define the MER for non-terminal states as $\Big[\operatorname{diag}\big(\exp(-\mbf r_N) / \lambda\big) - \mbf A_{NN} \Big]\inv$. A more general definition of the MER for all states is $\Big[\operatorname{diag}\big(\exp(-\mbf r) / \lambda\big) - \mbf A \Big]\inv$.

\section{More Experiment Details}
\label{appendix:exp-detail}
We provide additional experiment details omitted from the main text due to space constraints.\footnote{The code for this paper can be found at: \texttt{github.com/httse9/Reward-Aware-Proto-Representations}}

\subsection{Numerical Considerations}
\label{appendix:numerical_considerations}
Reward shaping and option discovery experiments involve the top eigenvector of the DR. We describe numerical considerations when performing eigendecomposition of the DR.
First, in practice, we perform eigendecomposition of the symmetrized DR, $\operatorname{Sym}(\mbf \Zeta) = (\mbf \Zeta + \mbf \Zeta^\top) / 2$, to ensure real eigenvalues and eigenvectors.
Second, as the DR involves exponentiating negative rewards (see Eq.~\ref{eq:DR_expanded}), the magnitude of the DR entries, especially the off-diagonal entries, can become very small. To mitigate the resulting numerical issues, we use the library $\texttt{python-flint}$.\footnote{https://github.com/flintlib/python-flint, MIT license} Note, however, that the improved precision comes at the cost of increased runtime. {As a larger $\lambda$ has the effect of reducing the magnitude of negative trajectory returns in the DR (see Eq.~\ref{eq:DR_expanded}), it can alleviate numerical instabilities. In our experiments, we initially started with $\lambda=1$, and slowly increased $\lambda$ by 0.1 until we settled on $\lambda=1.3$, which, paired with the use of $\texttt{python-flint}$, resulted in no numerical issues.}

Finally, the magnitude of the top eigenvector entries can be very small. Therefore, in practice, we use the logarithm of the top eigenvector in place of the top eigenvector. We now show that under mild assumptions, the top eigenvector of the DR is positive, allowing us to take the logarithm. We first present a mild assumption:

\renewcommand{\thetheorem}{\thesection.\arabic{theorem}}    
\begin{assumption}
    \label{assumption:reachable}
    There is only one start state, $s_0$, and it is possible to reach any state from $s_0$ under the default policy, i.e., $\exists \tau_{s_0\to s} : \mbf P^{\pi_d}(\tau_{s_0 \to s}) > 0$ for all $s \in\mathcal{S}$.
\end{assumption}

To avoid introducing unnecessary bias, we use the uniform random policy as the default policy in all of our experiments, which is the standard practice~\citep{piray2021linear, bazarjani_piray_2025}. 
We also perform experiments in grid world environments, where it is possible to reach any state from the start state.
Under these conditions, Assumption~\ref{assumption:reachable} easily holds. We now present the following proposition: \looseness=-1

\begin{proposition}
\label{theorem:top_eigenvector_positive_closed}
    Under Assumption~\ref{assumption:reachable}, the top eigenvector of $\operatorname{Sym}(\mbf \Zeta)$ is positive.
\end{proposition}

\begin{proof}
    We can rearrange the rows of $\mbf \Zeta$ so that the first row corresponds to the start state, $s_0$. By Assumption~\ref{assumption:reachable} and Eq.~\ref{eq:DR_expanded}, the first row of $\mbf \Zeta$ has all positive entries, i.e., $\mbf \Zeta(s_0, s) > 0\ \forall s \in \mathcal{S}$, while all other entries are non-negative. After symmetrization, $\operatorname{Sym}(\mbf \Zeta)$ has positive first row and first column. It is not hard to see, then, that $\operatorname{Sym}(\mbf \Zeta)^2$ has all positive entries, and $\operatorname{Sym}(\mbf \Zeta)$ is a primitive matrix~\citep{pillai2005perron}. Finally, by the Perron's theorem~\citep{pillai2005perron}, the largest eigenvalue of $\operatorname{Sym}(\mbf Z)$ is positive, and the corresponding eigenvector is positive.
\end{proof}

Proposition~\ref{theorem:top_eigenvector_positive_closed} applies to the case when the DR is computed in closed form. We now extend the proposition to the case when the DR is learned by TD learning. When learning the DR using TD learning (Eq.~\ref{eq:DR_TD}), the agent interacts with the environment and learn the DR in an incremental, online manner. Let $\mathcal{S}_V$ be the set of states visited by the agent. In this setting, the top eigenvector refers to the top eigenvector of the sub-matrix of the symmetrized DR corresponding to the states in $\mathcal{S}_V$. \looseness=-1

\begin{proposition}
    \label{theorem:top_eigenvector_positive_td}
    Assume that there is only one start state $s_0$. Let $\mathcal{D}$ be a dataset containing the transitions collected by an agent starting from $s_0$ and following the default policy. Importantly, the transitions are stored in the order in which they were collected. Initialize $\mbf \Zeta$ as the identity matrix. When learning the DR by TD learning (Eq.~\ref{eq:DR_TD}) with a step size $\alpha \in (0, 1)$, a backward sweep through $\mathcal{D}$ guarantees that the top eigenvector of $\operatorname{Sym}(\mbf \Zeta_{VV})$ is positive, where $\mbf \Zeta_{VV}$ is the sub-matrix of $\mbf \Zeta$ that corresponds to the states in $\mathcal{S}_V$.
\end{proposition}

\begin{proof}
    We can rearrange the rows of $\mbf \Zeta_{VV}$ so that the first row corresponds to the start state, $s_0$.
    The dataset, $\mathcal{D}$, consists of one or more trajectories starting from $s_0$. When iterating through a trajectory in a backward manner, for the row in $\mbf Z_{VV}$ corresponding to a state, $s$, all the entries corresponding to states visited after $s$ in the trajectory will be updated to have positive values.
    It is then not hard to see that after iterating through all transitions, the first row of $\mbf \Zeta_{VV}$ has all positive entries, while the remaining rows are non-negative. After symmetrization, $\operatorname{Sym}(\mbf \Zeta_{VV})$ has positive first row and first column. It is then easy to see that $\operatorname{Sym}(\mbf \Zeta_{VV})$ is a primitive matrix and, by Perron's theorem~\citep{pillai2005perron}, its top eigenvector is positive. \looseness=-1
\end{proof}

For MDPs with an initial distribution, $p_0$, over multiple start states, we can define a new start state that transitions to these start states under $p_0$, no matter what action is taken. The reward for this new start state also does not affect the optimal policy, and only needs to be negative (see Theorem~\ref{theorem:dp_convergence_main}). Proposition~\ref{theorem:top_eigenvector_positive_closed} and~\ref{theorem:top_eigenvector_positive_td} can then be applied.

In practice, when learning the DR using TD learning, for example, in RACE (Algorithm~\ref{alg:race}), we initialize the DR as the identity matrix, and perform at least one backward sweep through the dataset of collected transitions. We then compute the eigendecomposition for $\operatorname{Sym(\mbf \Zeta_{VV})}$, and take the logarithm of the resulting top eigenvector. To project this transformed eigenvector in $\mathbb{R}^{|\mathcal{S}_V|}$ back to $\mathbb{R}^{|\mathcal{S}|}$, we simply augment this vector with zeros for states in $\mathcal{S} \setminus \mathcal{S}_V$. 

\subsection{Reward Shaping}
\label{appendix:reward_shaping}

Figure~\ref{fig:reward_shaping_no_low_reward} shows reward-shaping results in the variations of the environments shown in Figure~\ref{fig:environments} without low-reward regions. In these environments, we do not observe a significant performance difference between potential-based reward shaping using the DR and the SR. This is because in these environments, all states apart from the terminal state have the same reward. Although Theorem~\ref{theorem:SR_DR_equiv_main} cannot be directly applied since the terminal state has a different reward ($0$) than the rest of the states ($-1$), it is reasonable to expect that the DR and the SR have similar eigenvectors in this setting, leading to similar results. \looseness=-1

\begin{figure}
    \centering
    \includegraphics[width=\linewidth]{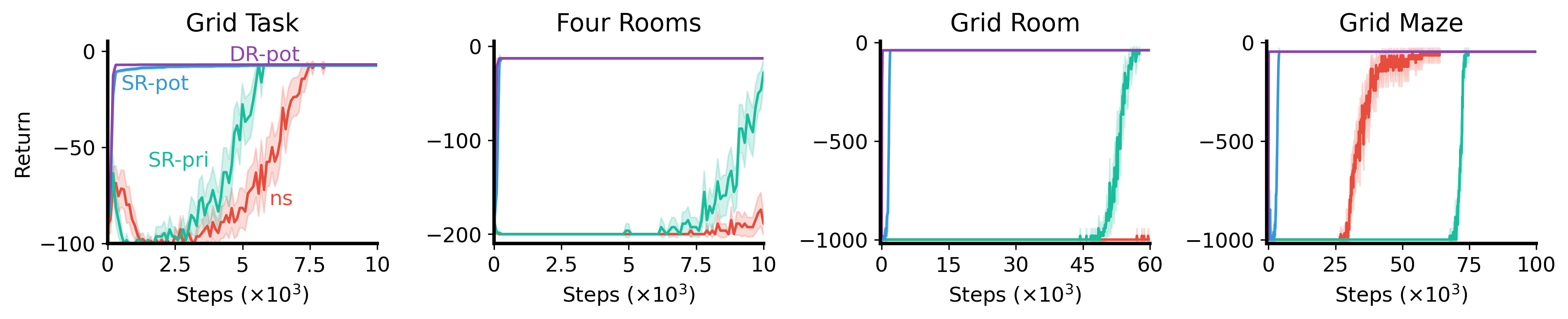}
    \caption{The average undiscounted return of potential-based reward shaping using the DR (\texttt{DR-pot}) and the SR (\texttt{SR-pot}), the prior distance-based reward shaping using the SR (\texttt{SR-prior})~\citep{wang2021towards}, and no shaping (\texttt{ns}) over 50 independent runs in the variations of the environments shown in Figure~\ref{fig:environments} without low-reward regions. The shaded area indicates 95\% confidence interval.}
    \label{fig:reward_shaping_no_low_reward}
\end{figure}

\subsection{Option Discovery}
\label{appendix:eigenoption_discovery}

In Algorithm~\ref{alg:race}, we present the full algorithm of RACE, an instance of the framework of representation-driven option discovery~\citep[ROD;][]{machado2023temporal} that relies on the DR instead of the SR for option discovery. 
Note that by learning the SR instead of the DR in Algorithm~\ref{alg:race}, we recover CEO~\citep{machado2023temporal}, and therefore omit the corresponding pseudocode for brevity.
In this version of RACE, given the complexity of performing importance sampling with options, we treat transitions generated from options the same as those generated from the default policy, and use both for learning the DR. \looseness=-1

We share the following hyperparameters for RACE and CEO: $\alpha_0=0.1$, $ \gamma_0=0.99$, $N_\text{step}=100$, $N_\text{iter}=50$ for grid task and four rooms, and $N_\text{iter} = 120$ for grid room and grid maze.
We use $\gamma=0.99$ for learning the SR, $\lambda=1.3$ for learning the DR, and sweep over the remaining hyperparameters, as described in Table~\ref{table:eigenoption_discovery_hyperparameters}. We perform 10 independent runs for each hyperparameter setting.

\begin{figure}[t]
\begin{algorithm}[H]
\caption{Reward-Aware Covering Eigenoptions (RACE)}
\label{alg:race}
\begin{algorithmic}
    \State {\bfseries Input:} $\alpha$ ; \Comment{Step size for learning the DR}
    \State \quad\quad\quad $\lambda$ ; \Comment{Relative importance of the deviation cost for the DR}
    \State \quad\quad\quad $N_\text{learn}$ ; \Comment{Number of iterations through collected dataset to learn the DR}
    \State \quad\quad\quad $N_\text{option}$ ; \Comment{Number of latest eigenoptions to keep}
    \State \quad\quad\quad $p_\text{option}$ ; \Comment{Probability of sampling an option instead of a primitive action}
    \State \quad\quad\quad $\alpha_0, \gamma_0$ ; \Comment{Step size and discount factor for learning the options' policies}
    \State \quad\quad\quad $N_\text{steps}$ ; \Comment{Number of interactions with the environment in each ROD iteration}
    \State \quad\quad\quad $N_\text{iter}$ ;  \Comment{Number of iterations of the ROD cycle}
    \State $\mathcal{D} \longleftarrow \emptyset$
    \State $\Omega \longleftarrow \emptyset$
    \State $\mbf \Zeta \longleftarrow \mbf I$
    \For{$i \longleftarrow 0$ to $N_\text{iter}$}
    \State \(\triangleright\) Collect samples
    \State $\mathcal{D}_\text{curr} \longleftarrow \emptyset$ 
    \For{$j \longleftarrow 0$ \textbf{to} $N_\text{steps}$}
        \State {With prob. $1 - p_\text{option}$ randomly sample, uniformly, a primitive action $a$; otherwise}
        \State \quad uniformly sample an option $\omega$ from $\Omega$
        \If{primitive action was sampled}
            \State In state $s$, take action $a$ and observe state $s'$ and reward $r$
            \State $\mathcal{D}_\text{curr} \longleftarrow \mathcal{D}_\text{curr} \| (s, a, r, s')$ \Comment{Append transition to current dataset}
        \Else{}
            \While{option $\omega$ not terminated}
                \State In state $s$, take action $a$ sampled under the option policy, 
                \State \quad and observe state $s'$ and reward $r$
                \State $\mathcal{D}_\text{curr} \longleftarrow \mathcal{D}_\text{curr} \| (s, a, r, s')$ \Comment{Append transition to current dataset}
            \EndWhile
        \EndIf
    \EndFor
    \State \(\triangleright\) Learn the default representation
    \State Update $\mbf \Zeta$ by $N_\text{learn}$ sweeps through $\mathcal{D}_\text{curr}$ using Equation~\ref{eq:DR_TD} with step size $\alpha$ 
    \State \quad and deviation importance $\lambda$
    \State \(\triangleright\) Learn eigenoption
    \State $\mathcal{D} \longleftarrow \mathcal{D} \| \mathcal{D}_\text{curr}$
    \State Compute the top eigenvector $\mbf e$ of $\mbf \Zeta$
    \State Use Q-learning to compute an eigenoption $\omega'$ using $\mbf e$, $\mathcal{D}$, $\alpha_0$, and $\gamma_0$
    \State \(\triangleright\) Add the learned eigenoption to the set of eigenoptions
    \State Add $w'$ to $\Omega$, and remove the oldest option from $\Omega$ if $|\Omega| > N_\text{option}$
   \EndFor
\end{algorithmic}
\end{algorithm}
\end{figure}

\begin{table}[t]
  \caption{Hyperparameter search for eigenoption discovery.}
  \label{table:eigenoption_discovery_hyperparameters}
  \centering
  \rowcolors{2}{white}{tablegray}
  \begin{tabular}{lll}
    \toprule
    \textbf{Name}     & \textbf{Description} & \textbf{Values} \\
    \midrule
    $p_\text{option}$ & Probability of selecting an option  & $[0.01, 0.05, 0.1]$    \\
    $N_\text{learn}$     & Number of iterations through the collected dataset to learn the SR/DR & $[1, 10, 100]$      \\
    $\alpha$  & Step size for learning the SR/DR       & $[0.01, 0.03, 0.1]$  \\
    $N_\text{option}$ &  Number of latest options to keep & $[1, 8, 1000]$ \\
    \bottomrule
  \end{tabular}
\end{table}

In Figure~\ref{fig:eigoption_discovery_scatter}, we used solid dots to highlight the hyperparameter settings that lead to the highest state visitation percentage. We show the learning curves for these hyperparameter settings in Figure~\ref{fig:eigenoption_discovery_curve}. We can see that RACE maintains a similar state visitation percentage as CEO, but obtains much higher rewards when exploring the environments. This is because RACE, being reward-aware, avoids low-reward regions when exploring the environments. The random walk, on the other hand, struggles to explore the state space, especially for the larger environments like grid room and grid maze, and does not encounter low-reward regions, causing it to obtain even higher cumulative rewards than RACE. \looseness=-1
\begin{figure}
    \centering
    \includegraphics[width=\linewidth]{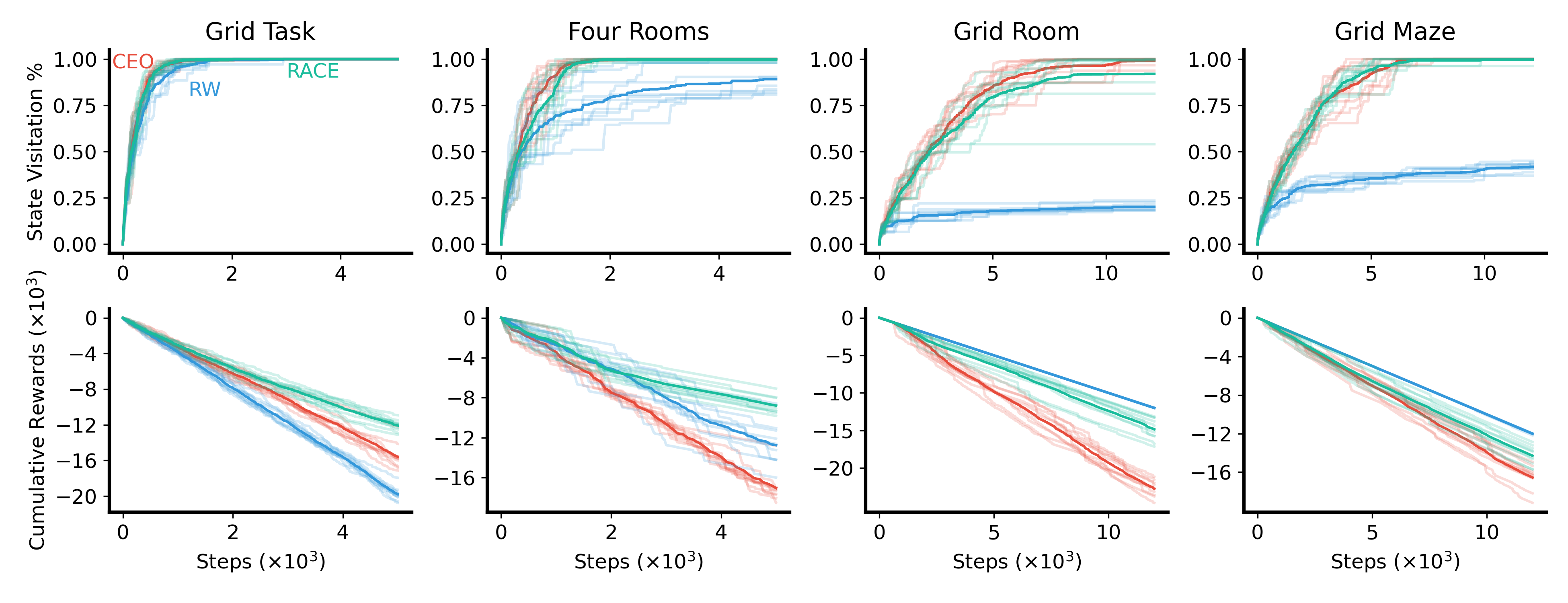}
    \caption{The state visitation percentage (top) and cumulative rewards (bottom) for the highlighted hyperparameter settings in Figure~\ref{fig:eigoption_discovery_scatter} of iterative online eigenoption discovery using the SR (CEO) and the DR (RACE), and the random walk (RW). The solid line shows the average over 10 seeds, while the individual seeds are shown in lighter shade.}
    \label{fig:eigenoption_discovery_curve}
\end{figure}

Figure~\ref{fig:eigenoption_discovery_without_lava_curve} shows the state visitation percentage for CEO, RACE, and the random walk in the variations of the environments in Figure~\ref{fig:environments} without any low-reward regions. As the reward function is constant, the top eigenvectors of the SR and the DR are guaranteed by Theorem~\ref{theorem:SR_DR_equiv_main} to be identical. It thus does not come as a surprise that we observe very similar state visitation percentages for CEO and RACE. 
\begin{figure}
    \centering
    \includegraphics[width=\linewidth]{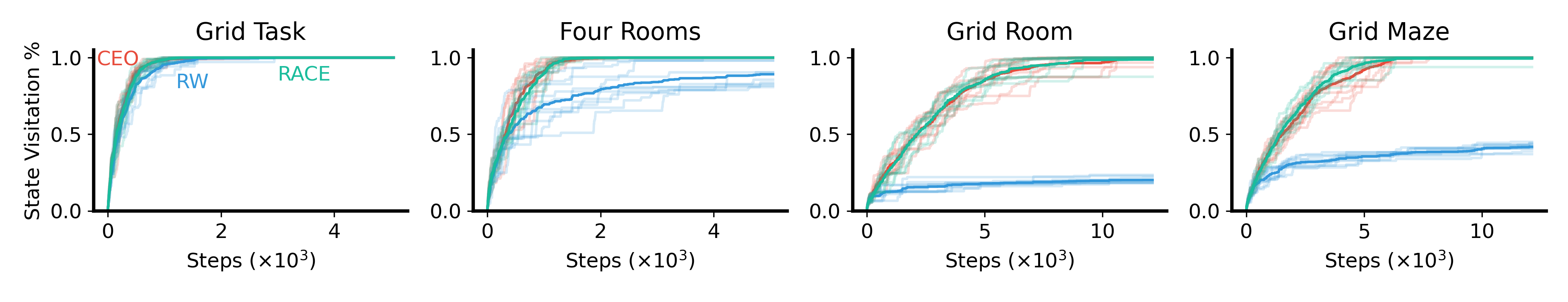}
    \caption{State visitation percentage for CEO, RACE, and the random walk (RW) in the environments shown in Figure~\ref{fig:environments} without low-reward regions. Shown in lighter shade are the 10 individual seeds. \looseness=-1}
    \label{fig:eigenoption_discovery_without_lava_curve}
\end{figure}

We now describe details on combining iterative online eigenoption discovery with Q-learning~\citep{watkins1989learning}.
Specifically, we use iterative online eigenoption discovery to collect transition data, and then perform offline Q-learning using the collected data. We compare RACE with Q-learning (RACE+Q) with CEO with Q-learning (CEO+Q), and a Q-learning baseline (QL). For the Q-learning baseline, at every iteration, we use the $\epsilon$-greedy policy induced by the current Q-values to collect transitions, and then update the Q-values using the collected transitions.

For RACE+Q and CEO+Q, we follow the same hyperparameter search procedure as before (see Table~\ref{table:eigenoption_discovery_hyperparameters}). 
For the Q-learning baseline, we perform a grid search over the Q-value initialization ($[-1000, -100, -10, 0]$), $\epsilon$ for $\epsilon$-greedy exploration ($[0.01, 0.05, 0.1, 0.15, 0.2]$), and step size ($[0.01, 0.03, 0.1, 0.3, 1]$). We run 10 seeds for each hyperparameter setting, and after identifying the best hyperparameters, re-run 50 seeds to avoid maximization bias.
As we do not see much performance difference in simpler environments, we introduce two new environments that are larger versions of Grid Room (Grid Room (L)) and Grid Maze (Grid Maze (L)), shown in Figure~\ref{fig:env_larger}. 

\begin{figure}
    \centering
    \includegraphics[width=\linewidth]{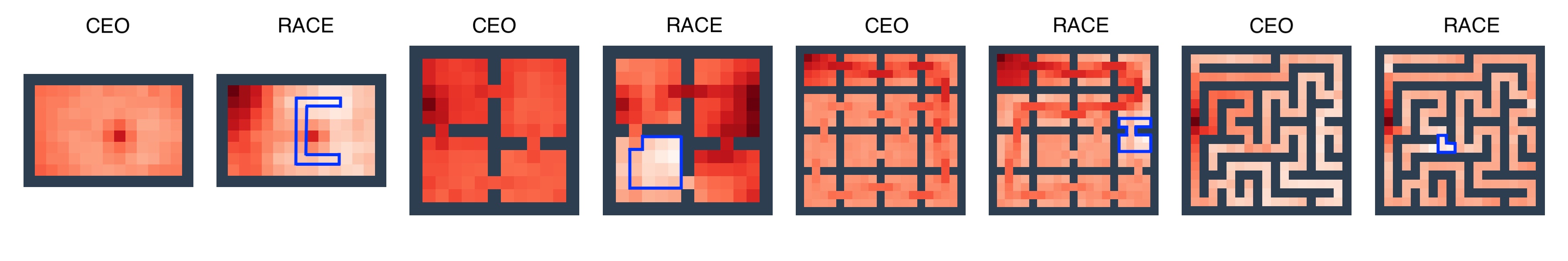}
    \caption{Cumulative state visits for CEO and RACE in the environments from Figure~\ref{fig:environments} averaged over 10 seeds, where darker red indicates more visits. The low-reward regions are enclosed in blue. Although CEO is also applied to the problem with low-reward regions, we emphasized them only for RACE because CEO does not see it. As indicated by the lighter red, RACE visits low-reward regions much less than CEO. It is especially clear in grid room (third env. from the left) that RACE takes detours to visit the bottom rooms without passing through low-reward regions.}
    \label{fig:eigenoption_discovery_cumulative_visit}
\end{figure}

\begin{figure}
    \centering
    \includegraphics[width=0.5\linewidth]{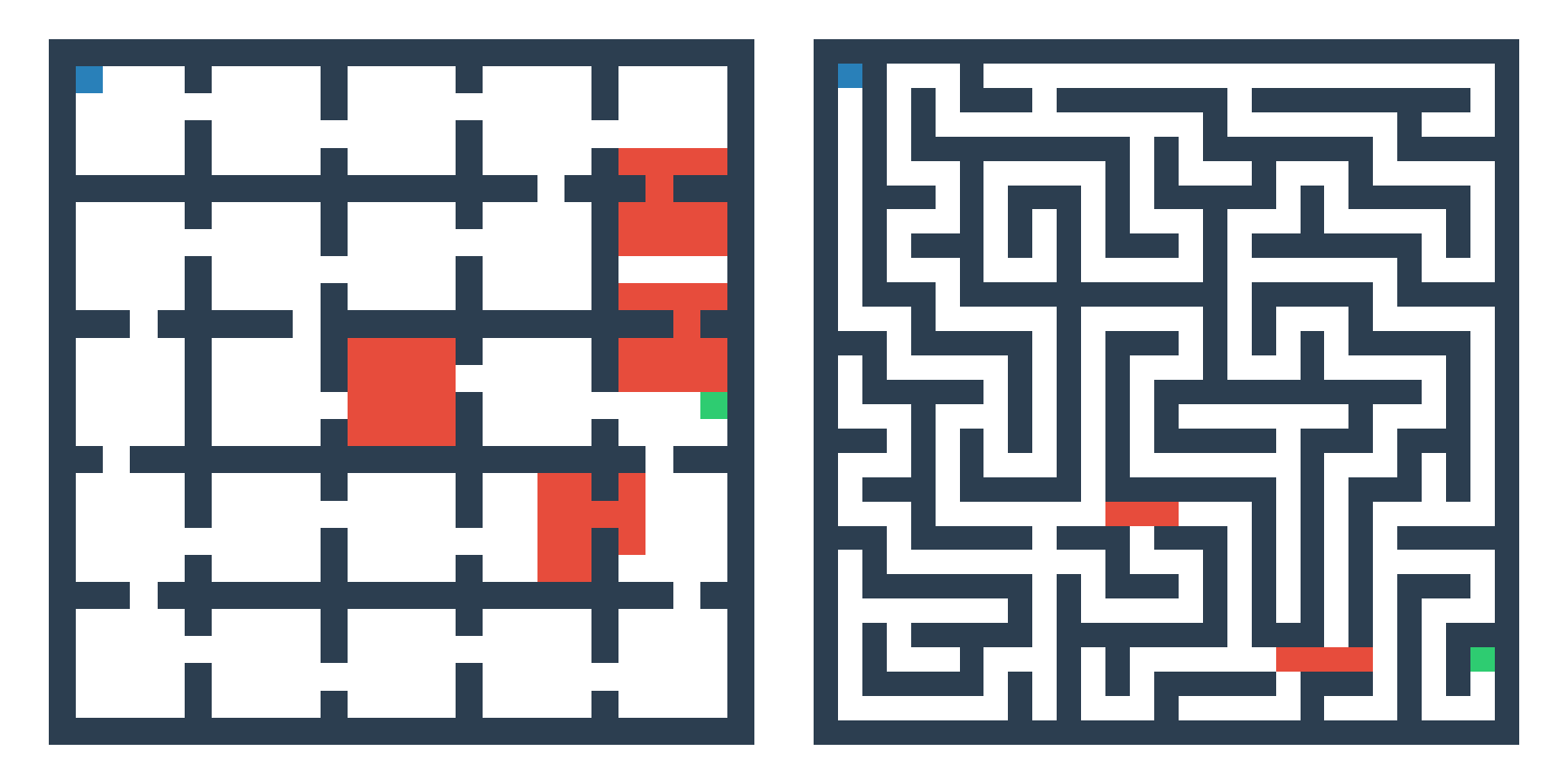}
    \caption{Left: Larger version of Grid Room (Grid Room (L)); Right: Larger version of Grid Maze (Grid Maze (L)), adapted from prior work~\citep{wang2023reachability}.}
    \label{fig:env_larger}
\end{figure}

\newpage
\subsection{Count-Based Exploration}
\label{appendix:count-based-exploration}
To facilitate the learning and application of the DR for exploration in the Riverswim and Sixarms environments, we rescale the environment rewards to lie within the range $[-1, 0]$. 
Importantly, this rescaling is applied solely for learning the DR; the Sarsa agent continues to operate using the original, unscaled rewards for action-value estimation. Note that we use the DR defined for state-action-dependent rewards, $\bar{\mbf \Zeta}$.

For Sarsa and Sarsa+SR, we adopt the code (MIT license) and hyperparameters specified by the original authors~\citep{machado2020count}. It is to be noted that while they use $r_{\text{intr}}(s) = \beta \cdot \frac{1}{\| \mbf \Psi^\pi_{s,:} \|_1} $ as the intrinsic reward~\citep{machado2020count}, we use $r_{\text{intr}}(s, a) = \beta \cdot \log( \| \bar{\mbf \Zeta}_{sa,:} \|_2)$. We empirically find that the logarithmic transformation enhances performance when leveraging the DR for exploration.

For Sarsa+DR, we sweep over different values of $\eta, \alpha, \beta, \lambda$, with $\eta \in \{ 0.01,0.1,0.25,0.5\}$, $\alpha \in \{ 0.01,0.1,0.25,0.5\}$, $\beta \in \{ 0.1, 1, 10, 100\}$ and $\lambda \in \{1, 1.5, 2\}$. Here, $\eta$ and $\alpha$ denote the step sizes for updating the Q-values and the DR, respectively. $\beta$ is the scaling factor for the intrinsic reward and $\lambda$ controls the {relative importance of the deviation cost}. We use $\epsilon=0.01$ for $\epsilon$-greedy exploration. Additionally, we experimented with different transformations of the DR for computing the intrinsic reward, specifically, $\operatorname{transform}(\mbf x) \in \{ \| \mbf x \|_1, \| \mbf x \|_2, \log(\| \mbf x\|_1), \log(\|\mbf x\|_2) \}$, where $\mbf x$ is a row of $\bar{\mbf \Zeta}$. The best hyperparameters, shown in Table~\ref{table:count_based_hypers}, were evaluated on 100 independent runs.

\begin{table}
\centering
\caption{Hyperparameters for Count-Based Exploration (Sarsa + DR).}
\label{table:count_based_hypers}
\small
\rowcolors{2}{white}{tablegray}
\begin{tabular}{@{}lccccc@{}}
\toprule
\textbf{Environment} & $\alpha$ & $\eta$ & $\beta$ & $\lambda$ & $\operatorname{{transform}}(\mbf x)$\\
\midrule
\textsc{RiverSwim} & 0.5 & 0.25 & 100 & 1 & $\log(\|\mbf x\|_2)$\\
\textsc{SixArms}   & 0.5 & 0.01 & 0.1 & 1.5 & $\log(\|\mbf x\|_2)$  \\
\bottomrule
\end{tabular}
\end{table}

\subsection{Transfer}
\label{appendix:transfer}
We describe the details of the features used for transfer learning. We assume that the agent has access to features that can perfectly represent terminal rewards for the DFs, and features that can perfectly represent the reward function for the SFs. For the DR, as there are four terminal states in the environment, and we assume the agent receives the same reward at a terminal state regardless of which action is chosen, we represent the features as a 4-dimensional one-hot vector, where each entry in the vector corresponds to one terminal state.

For the SR, we have to represent the reward function for all states. There are three types of states in the environment: empty state (indicated by white tile), low-reward state (indicated by red tile), and goal state (indicated by green tile). All empty states have the same reward of $-1$, and all low-reward states have the same reward of $-20$. Therefore, we represent state features as a 6-dimensional one-hot vector, where one entry is activated when the state is an empty state, one entry is activated when the state is a low-reward state, and the remaining four entries correspond to the four terminal states. \looseness=-1

Given the features, we compute the DFs and the SFs under randomly sampled configurations of terminal state rewards, where the reward at each terminal state is sampled independently from a normal distribution with 0 mean and 50 standard deviation. 
We learn the DFs by following the default policy. While we only need to learn the DFs under the default policy, we can learn the SFs with respect to a set of policies learned under different reward functions. The better the set of reward functions covers the space of reward functions, the better the transfer policy computed using the SFs. We consider the number of reward functions to be 1, 2, 4, and 8. For each randomly sampled reward function, we compute the SFs with respect to the optimal policy learned under this sampled reward function. \looseness=-1

In this work, we ensure that the DFs and the SFs are well learned by allowing the agent to interact with the environment for a large number of steps. For the DFs, the agent interacts with the environment while following the default policy for 100K steps. For the SFs, we first use Q-learning~\citep{watkins1989learning} to learn an optimal policy for 100K steps. We then learn the SFs while following the optimal policy for another 100K steps. We use $\lambda=1.3$ and the uniform random policy as the default policy for the DFs, $\gamma=0.99$ for the SFs, and use a step size of $0.1$ for both approaches. We perform 50 independent runs. \looseness=-1

Note that the DFs computes an optimal policy in the linearly solvable MDP setting that balances the reward function and the cost of deviating from the default policy. The optimal policies computed by the DFs are then softly biased towards the default policy. We mitigate this bias by approximating the optimal transfer policies (see Eq.~\ref{eq:DR_optimal_policy}) using deterministic policies greedy over the optimal Q-values. \looseness=-1

\section{Statistical Significance}
\label{appendix:statistical_significance}
We report 95\% confidence intervals or directly visualize all independent runs if possible. The 95\% confidence intervals capture randomness across independent runs, where the primary source of randomness is random exploration, e.g., $\epsilon$-greedy exploration. In transfer learning experiments, randomness also arises from sampling terminal state reward configurations. As we perform a large number of independent runs in our experiments ($N \geq 50$), we assume the sampling distribution is normal and compute the 95\% confidence interval as $\pm 1.96 \frac{\hat\sigma}{\sqrt{N}}$, where $\hat\sigma$ is the sample standard deviation. \looseness=-1

\section{Compute Resources}
\label{appendix:compute_resources}
We use CPUs for all of our experiments. 
We describe the runtimes for experiments involving the DR. For reward-shaping experiments, each independent run takes under 10 minutes. For eigenoption discovery experiments, each independent run takes less than 2 hours in grid task and four rooms, and takes around 10 hours in grid room and grid maze. For count-based exploration experiments, each independent run takes less than one minute. For transfer experiments, each independent run takes less than 5 minutes. Due to the large number of independent runs performed for hyperparameter search and preliminary experiments, we estimate the total compute used for the project to be 10.5 CPU core years. \looseness=-1


\end{document}